\documentclass[10pt,twocolumn,letterpaper]{article}

\usepackage{cvpr}              % To produce the CAMERA-READY version
\usepackage{amsthm}
\theoremstyle{plain}

\newtheorem{lemma}{Lemma}

\newtheorem{corollary}{Corollary}
\usepackage{amsmath}
\usepackage{bm}
\usepackage{pifont}
\newcommand{\cmark}{\ding{51}}% checkmark
\newcommand{\xmark}{\ding{55}}% cross
\usepackage{amsmath}
\usepackage{amssymb}
\usepackage{array} 
\usepackage{makecell}
\usepackage{algorithm}
\usepackage{algpseudocode} % This is part of the algorithmicx package bundle

\usepackage{float}
\usepackage{placeins}

 \usepackage{booktabs} % Required for \toprule, \midrule, \bottomrule
\usepackage{booktabs}      % For professional-looking rules (\toprule, \midrule, \bottomrule)
\usepackage{siunitx}       % For aligning numbers in columns (S column type)
\usepackage{graphicx}      % For \resizebox and \rotatebox
\usepackage{multirow}      % For spanning rows (\multirow)
\usepackage[table]{xcolor} % For \rowcolor and \textcolor (the [table] option is important)
\usepackage{cuted}
\usepackage{stfloats}

\usepackage{placeins}

\definecolor{cvprblue}{rgb}{0.21,0.49,0.74}
\usepackage[pagebackref,breaklinks,colorlinks,allcolors=cvprblue]{hyperref}

\def\paperID{****} % *** Enter the Paper ID here
\def\confName{CVPR}
\def\confYear{2026}

\title{DSCA: Dynamic Subspace Concept Alignment for Lifelong VLM Editing}
\author{
Gyanendra Das$^{1}$ \quad Sai Satyam Jena$^{1}$\\
$^{1}$ Zynix AI, FL, USA\\
{\tt\small \{gyanendra, sai\}@zynix.ai}
}

\begin{document}
\maketitle
\begin{abstract}
Model editing aims to update knowledge to add new concepts and change relevant information without retraining. Lifelong editing is a challenging task, prone to disrupting previously learned concepts, especially for Vision Language Models (VLMs), because sequential edits can lead to degraded reasoning and cross-modal misalignment. Existing VLM knowledge editing methods based on gated adapters, activation edits, and parameter merging techniques address catastrophic forgetting seen in full fine-tuning; however, they still operate in the shared representation space of the VLM, where concepts are entangled, so edits interfere with other non-relevant concepts. We hypothesize that this instability persists because current methods algorithmically control edits via optimization rather than structurally separating knowledge. We introduce Dynamic Subspace Concept Alignment (DSCA) which by design mitigates this limitation by decomposing the representation space into a set of orthogonal semantic subspaces and proposing edits only in those transformed spaces. These subspaces are obtained through incremental clustering and PCA on joint vision-language representations. This process structurally isolates concepts, enabling precise, non-interfering edits by turning isolation from a soft training objective into an architectural property. The surgical edits are guided by a multi-term loss function for maintaining task fidelity, edit locality, and cross-modal alignment. With the base model frozen, our method achieves 98\% single-edit success, remains over 95\% after 1,000 sequential edits, lowers hallucination by 3-5\%, and achieves the best backward transfer (BWT) scores on continual instruction-tuning benchmarks. Extensive experiments demonstrate DSCA’s state-of-the-art stability and knowledge retention capability in continual lifelong editing across various datasets and benchmarks. 
\end{abstract}

\section{Introduction}
\label{sec:intro}
% \FloatBarrier
\begin{figure}[htbp] % requires \usepackage{float}
    \centering
    \includegraphics[width=\columnwidth]{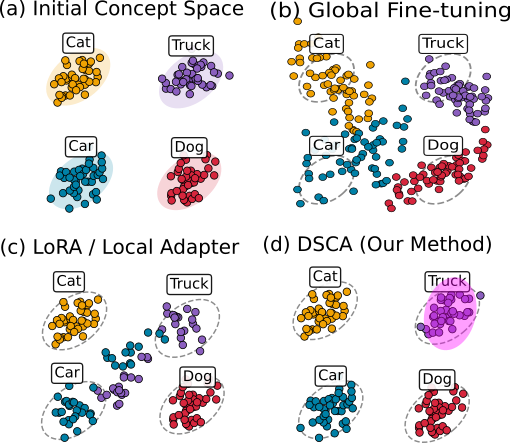}
    \vspace{-2mm}
    \caption{\textbf{Conceptual comparison of knowledge-editing paradigms.}
    (a) The initial concept space where concepts are well-separated.
    (b) Global fine-tuning perturbs the entire representation space, distorting unrelated concepts.
    (c) LoRA / local adapters constrain edits but still produce coupled interference.
    (d) \textbf{DSCA} performs subspace-confined, concept-specific interventions, maintaining isolation and preserving all other concepts.}
    \label{fig:overview}
    \vspace{-3mm}
\end{figure}
Large vision--language models (LVLMs) are increasingly deployed as long-lived systems that interact with users over months or years. In such settings, we cannot treat their knowledge as static; facts change, user-specific preferences evolve and model errors must be corrected without retraining from scratch. Humans learn new concepts in a modular fashion; learning about the ``Tesla Cybertruck'' does not alter one's concept of a ``Road''.
%This ability hinges on a cognitive architecture that allows for surgical,context-specific updates %without requiring a global rewrite.
This localized knowledge updating contrasts sharply with current VLMs, whose knowledge resides in a high-dimensional representation manifold where edits tend to cause coupled interference across concepts (Fig.~\ref{fig:overview}). Consequently, attempts to teach VLMs new concepts often trigger global perturbations. Full fine-tuning drastically alters the manifold's geometry, destroying the carefully learned relational structure between existing concepts and leading to catastrophic forgetting. Lighter-weight methods attempt to isolate edits from interfering with unrelated concepts in two main ways. Methods like LiveEdit~\cite{chen2025lifelong} and DualEdit~\cite{shi2024dualedit} use routing logic to activate small and selective ``expert'' modules for specific inputs. Others, such as PAM~\cite{sokar2025continual} and ConDU~\cite{gao2025enhanced}, learn parameters for a new task and then carefully merge them back into the base model's weights. However, both strategies still apply updates to the model parameters. Any alteration to model weights, even to a small subset, inevitably perturbs the shared representation space of the VLM. Thus, an edit intended for one concept can unintentionally shift the position of nearby representations, subtly distorting the model's understanding of unrelated but similar concepts.

Our core conviction is that this issue is not an algorithmic flaw to be patched, but a fundamental architectural mismatch. If knowledge in the real world is compositional and interventions are local, then edits should occur in the respective concept subspaces of the VLM rather than in the shared representation manifold. This vision requires a knowledge-editing mechanism whose architecture is \textit{plastic}, allowing the base model’s conceptual space to be extended and refined as new information is acquired.

This paper introduces \textbf{Dynamic Subspace Concept Alignment (DSCA)}, a framework built from the ground up on this principle. Rather than altering the model’s core weights, DSCA performs precise modifications directly within the relevant semantic subspace, with basis-level control \cite{liu2025unlocking}. Instead of treating the VLM’s representation space as a monolithic entity, DSCA decomposes it into a dynamic collection of orthogonal subspaces, each housing a distinct concept. This design creates structural \textquotedblleft firewalls\textquotedblright\ that prevent edits to one concept from interfering with others. This architectural shift enables models to adapt to new out-of-distribution data in a structured, robust, and human-like manner.
Our key contributions can be summarized as follows:
\begin{enumerate}
    \item \textbf{A novel editing architecture via subspace decomposition.} We introduce a method that structurally partitions the VLM’s representation space into a dynamic set of orthogonal semantic subspaces. This principled separation ensures edits are isolated by construction, eliminating cross-concept interference.
    \item \textbf{State-of-the-art reliability in lifelong learning.} DSCA demonstrates superior performance on both single-edit and sequential editing benchmarks. Notably, it maintains exceptional reliability and near-perfect locality locality after 1,000 sequential edits, proving its robustness in long-term scenarios where existing methods typically suffer catastrophic failure.
    \item \textbf{Enhancement of foundational VLM capabilities.} Our framework rigorously safeguards the base model, post-editing DSCA not only preserves performance on standard benchmarks (VQA-v2\cite{balanced_vqa_v2}, MME\cite{fu2023mme}) but also improves generalization and reduces hallucination rates by 3–5\% compared to existing editors.
    \item \textbf{A scalable and efficient intervention mechanism.} We present an efficient system that decouples rapid, task-specific learning from slower, data-driven structural refinement of concept subspaces. This design enables continuous assimilation of new information with minimal inference overhead, making DSCA practical for real-world model evolution.
\end{enumerate}

\section{Related Works}
\label{sec:RelatedWorks}
Continual learning for Vision-Language Models (VLMs) faces unique challenges; degraded cross-modal alignment, interference in shared pathways, and loss of zero-shot generalization \cite{liu2025continual}. Three main approaches have emerged: data replay, regularization, and architectural adaptation. We address limitations of the first two and advance the third with a novel activation-space intervention that achieves strong subspace-level architectural isolation. Our DSCA framework belongs to this architectural adaptation family, operating entirely in activation space rather than modifying base parameters.

\noindent\textbf{Multi-Modal Replay.} Replay methods revisit past data to prevent forgetting. Explicit methods store raw samples \cite{zhang2023vqacl}, while implicit methods use generative models to create synthetic samples \cite{yan2022generative, Frascaroli_2024_BMVC}, avoiding privacy issues and reducing storage. However, the computational cost of training and sampling from generative models limits scalability.

\noindent\textbf{Cross-Modal Regularization.} Regularization methods add constraints to protect existing knowledge without storing data. C-CLIP \cite{liu2025cclip} preserves embedding geometry, ZSCL \cite{Zheng_2023_ICCV} maintains similarity distributions, DualTeacher \cite{yu2024select} uses knowledge distillation, and Mod-X \cite{pmlr-v202-ni23c} regularizes similarity matrices. These are efficient but act as "soft" constraints that cannot guarantee architectural isolation, especially for related concepts.

\noindent\textbf{Parameter-Efficient Adaptation (PEA).} PEA methods freeze the base VLM and add minimal new parameters to limit forgetting. This paradigm has evolved from direct parameter modifications to activation-space interventions.

\noindent\textit{1) \textbf{Direct Parameter Modification.}} Methods insert lightweight modules into the VLM. Some merge task-specific LoRA \cite{hu2021lora} modules (PAM \cite{sokar2025continual}) or dynamically combine them during inference (CoDyRA \cite{lu2024adaptive}). Mixture-of-Experts approaches use learned gating to activate specific adapters, as in MoE-Adapters \cite{yu2024boosting} and DualEdit \cite{shi2024dualedit}. CLAP4CLIP \cite{jha2024clap4clip} uses probabilistic adapters to model task-specific distributions. LiveEdit \cite{chen2025lifelong} combines low-rank MoE with two-stage routing for selective edits.

\noindent\textit{\textbf{2) Canonical Model Editing.}} Within the broader model adaptation literature, a parallel line of work in large language models focuses on directly modifying model weights to encode factual knowledge. ROME \cite{meng2022locating} and MEMIT \cite{meng2023memit} update the MLP weights of specific layers to insert or correct facts without retraining. Gradient-based and memory-based editors such as MEND~\cite{mitchell2022fast}, SERAC~\cite{mitchell2022memory}, LTE~\cite{jiang2024learning}, and VisEdit~\cite{chen2025attribution} similarly operate in parameter space or external memories, and we show in Sec.~\ref{sec:experiments} that their performance degrades under long multimodal edit sequences.

\noindent\textit{\textbf{3) Modular Activation-Space Intervention.}} This paradigm manipulates a model's computational graph at inference time by altering its activations(ReFT\cite{wu2024reft} for LLM knowledge editing). We extend this philosophy from LLMs to Vision-Language Models (VLMs), introducing further modifications to handle multimodal representations and address the open challenge of ensuring structural isolation when multiple interventions are applied concurrently. Its limitation, as shown by \cite{liu2025unlocking}, is that this uniform update struggles to achieve both successful editing and locality simultaneously. BaFT \cite{liu2025unlocking} offers a more precise solution by making the intervention non-linear and input-dependent. By adaptively determining the update's magnitude along each basis direction of the subspace, BaFT can tailor the edit for each specific input, significantly improving the editing-locality trade-off. We extend this philosophy from LLMs to Vision-Language Models (VLMs), introducing further modifications to handle multimodal representations(Sec.\ref{sec:HierarchicalGating}) and address the open challenge of ensuring structural isolation when multiple interventions are applied concurrently.(Sec.~\ref{sec:concept_partitioning}, \ref{sec:dsam})

\noindent\textbf{Discussion.} DSCA complements these approaches by introducing \textit{architectural orthogonality at the representation level} (Sec.~\ref{sec:dsam}), achieving subspace-level isolation that bridges activation-space precision with structural modularity.
    
\section{Methodology}
\label{sec:method}

\subsection{Problem Formulation}
\label{sec:problem_formulation}

Given a pre-trained VLM $\mathcal{M}$ with frozen parameters $\theta$, we focus on the fused cross-modal representation
$\mathbf{h}_f = \mathcal{M}_\theta(I, T) \in \mathbb{R}^{d_f}$ for an image--text pair $(I, T)$, and denote unimodal visual and textual features as $\mathbf{h}_v \in \mathbb{R}^{d_v}$ and $\mathbf{h}_t \in \mathbb{R}^{d_t}$, with fusion
\[
\mathbf{h}_f = \text{Fuse}(\mathbf{h}_v, \mathbf{h}_t).
\]

We consider a sequence of edits $\mathcal{E} = \{E_1, E_2, \dots \}$ applied to a frozen backbone, where each edit $E_i$ specifies a desired change in behavior for a particular input $(I_e, T_e)$ (e.g., updating an outdated fact or adding a new concept).

Our goal is to learn an intervention function $\Psi$ operating directly in the representation space:
\begin{equation}
    \mathbf{h}'_f = \mathbf{h}_f + \Delta \mathbf{h}_f = \Psi(\mathbf{h}_f; \phi),
\end{equation}
where $\phi$ collects the parameters of the editing modules and $\Delta \mathbf{h}_f$ is the proposed update.

The intervention should satisfy three objectives:
\begin{enumerate}
    \item \textbf{Task Fidelity:} Edited representations $\mathbf{h}'_{f,e}$ must yield the desired output for edit samples $(I_e, T_e)$.
    \item \textbf{Locality:} For out-of-scope samples $(I_o, T_o)$, the intervention should be minimal, i.e.,
    $\Psi(\mathbf{h}_{f,o}; \phi) \approx \mathbf{h}_{f,o}$, preserving unrelated knowledge.
    \item \textbf{Cross-Modal Alignment:}
    Updates $\Delta \mathbf{h}_f$ should not disrupt consistency between visual and textual semantics.
\end{enumerate}

DSCA implements $\Psi$ via concept-specific semantic subspaces and sparsely routed modules that operate only where needed.

%--------------------------------------%

\subsection{Online Semantic Partitioning of the Representation Space}
\label{sec:concept_partitioning}

To achieve locality, DSCA first organizes the fused representation space into concept clusters. Incoming fused features are assigned online to an evolving set of $K$ clusters $\{C_1, \dots, C_K\}$.

Each cluster $C_k$ is represented by a fused prototype
$\mathbf{p}_{k,f} \in \mathbb{R}^{d_f}$, updated by an exponential moving average (EMA) over assigned features. Given a new fused representation $\mathbf{h}_f$, we first associate it with the nearest prototype
\begin{equation}
\label{eq:findNearestFusedPrototype}
    j = \arg\min_{k \in \{1,\dots,K\}} \|\mathbf{h}_f - \mathbf{p}_{k,f}\|_2.
\end{equation}

To detect genuinely novel concepts, we maintain per-cluster statistics $(\mu_j, \sigma_j)$ over distances to $\mathbf{p}_{j,f}$ and define a dynamic threshold
\[
d_j = \mu_j + \alpha \cdot \sigma_j,
\]
with sensitivity hyperparameter $\alpha$. If
\[
\left\lVert \mathbf{h}_f - \mathbf{p}_{j,f} \right\rVert_2 > d_j,
\]
we instantiate a new cluster $C_{K+1}$ with $\mathbf{h}_f$ as its first member; otherwise, $\mathbf{h}_f$ is assigned to $C_j$ and both prototype and statistics are updated.
During training, clusters expand as new data arrive; at inference time they are frozen and used as an efficient routing index so that edits are applied only to relevant regions of the representation space.

%--------------------------------------%

\subsection{Dynamic Structured Alignment Modules (DSAMs)}
\label{sec:dsam}

For each concept cluster $C_k$, DSCA attaches a \emph{Dynamic Structured Alignment Module} (DSAM) that proposes a concept-specific update. Each DSAM consists of: (1) a semantic subspace, (2) a learnable transformation within that subspace, and (3) an input-dependent gating mechanism.\\
\textbf{(1) Semantic Subspace $R_k$.}
Performing edits directly in the full $d_f$-dimensional fused space is both expensive and brittle. Instead, we introduce a low-rank subspace for each concept,
\[
R_k \in \mathbb{R}^{r \times d_f}, \quad r \ll d_f,
\]
whose rows span the principal axes of variation for features in $C_k$.
Crucially, $R_k$ is not updated by backpropagation. It is:
\begin{itemize}
    \item initialized via PCA once $C_k$ has accumulated at least $N_{\text{min}}$ samples, and
    \item periodically refined using Incremental PCA on new samples assigned to $C_k$.
\end{itemize}
In practice we apply Incremental PCA on residualized features with respect to earlier subspaces, keeping the family $\{R_k\}$ approximately orthogonal (see Supplementary for details and analysis).\\
\textbf{(2) Learnable Subspace Transformation $(W_k, b_k)$.}
Given $\mathbf{h}_f$, DSAM $k$ predicts target coordinates within its subspace via
a linear transformation
\[
W_k \in \mathbb{R}^{r \times d_f}, \quad b_k \in \mathbb{R}^{r}.
\]
$W_k$ maps the high-dimensional fused feature into the $r$-dimensional semantic basis, and $b_k$ shifts it toward the new conceptual center induced by edit data. The term
$(W_k \mathbf{h}_f + b_k)$ encodes the desired coordinates for $\mathbf{h}_f$ in that subspace.\\
\textbf{(3) Component-wise Gating $\Gamma_k(\mathbf{h}_f)$.}
The raw update proposed by DSAM $k$ is computed as a residual in the subspace and then lifted back to the full space:
\begin{equation}
\label{eq:editindf}
\Delta \mathbf{h}_f = \mathbf{R}_k^\top \big( (\mathbf{W}_k \mathbf{h}_f + \mathbf{b}_k) - \mathbf{R}_k \mathbf{h}_f \big).
\end{equation}
To ensure minimal, input-specific changes, we introduce an input-dependent gating function
$\gamma_k(\mathbf{h}_f) \in [0,1]^{d_f}$, parameterized by a lightweight neural layer:
\begin{equation}
 \label{eq:gamma}
     \gamma_k(\mathbf{h}_f) = \sigma(W_{g,k}\mathbf{h}_f + b_{g,k}),
\end{equation}
where $W_{g,k}$ and $b_{g,k}$ are learnable and $\sigma(\cdot)$ is the element-wise sigmoid. The gating vector defines a diagonal matrix
\begin{equation}
 \label{eq:diag}
     \Gamma_k(\mathbf{h}_f) = \text{diag}(\gamma_k(\mathbf{h}_f)),
\end{equation}
which selectively attenuates dimensions of $\Delta \mathbf{h}_f$. This yields a fine-grained, input-adaptive correction. (We implement $W_{g,k}$ via a low-rank bottleneck for efficiency; see Supplementary) %Sec.~\ref{sec:supp_gating_details}.)

%--------------------------------------%

\subsection{Two-Stage Hierarchical Routing}
\label{sec:HierarchicalGating}

Evaluating all $K$ DSAMs per input would be inefficient. DSCA therefore uses a two-stage hierarchical routing mechanism that first performs coarse visual filtering and then fine-grained routing in the fused space.

For each concept $C_k$ we maintain
a visual prototype $\mathbf{p}_{k,v} \in \mathbb{R}^{d_v}$ (EMA of visual features) and the fused prototype $\mathbf{p}_{k,f}$ from Sec.~\ref{sec:concept_partitioning}.

\paragraph{Stage 1: Coarse visual filtering.}
Given visual features $\mathbf{h}_v$, we compute cosine similarities with all $\{\mathbf{p}_{k,v}\}$ and retain only concepts above a threshold $\tau_{\text{visual}}$:
\begin{equation}
C_{\text{cand}} = \left\{ k \;\middle|\; \frac{\mathbf{h}_v^\top \mathbf{p}_{k,v}}{\|\mathbf{h}_v\|_2 \|\mathbf{p}_{k,v}\|_2} > \tau_{\text{visual}} \right\}.
\label{eq:coarse_filter}
\end{equation}
This provides a small candidate set of potentially relevant DSAMs.

\paragraph{Stage 2: Fused routing.}
For each candidate $k \in C_{\text{cand}}$, we compute a similarity score
$s_k = \text{cosine\_similarity}(\mathbf{h}_f, \mathbf{p}_{k,f})$ and convert these into normalized routing weights via a temperature-controlled softmax:
\begin{equation}
w_k = 
\begin{cases} 
  \frac{\exp(s_k/\tau)}{\sum_{j \in C_{\text{cand}}} \exp(s_j/\tau)} & \text{if } k \in C_{\text{cand}}, \\
  0 & \text{otherwise},
\end{cases}
\label{eq:fine_routing}
\end{equation}
where $\tau$ is a temperature hyperparameter. We denote the corresponding logits $z_k = s_k/\tau$ for $k \in C_{\text{cand}}$ (and $z_k = 0$ otherwise); these logits are also used by the sparsity loss in Sec.~\ref{sec:loss}.
\begin{figure*}[htbp]
\begin{minipage}[t]{0.5\textwidth}
    \begin{algorithm}[H]
    \caption{DSCA Continual Training Loop}
    \label{alg:main_loop}
    \begin{algorithmic}[1]
    \State \textbf{Input:} Pretrained VLM $f_\theta$, edit stream $\mathcal{D}_e$, replay data $\mathcal{D}_o$.
    \State \textbf{Hyperparams:} $\lambda_{\text{align}}, \lambda_{\text{distill}}, \lambda_{\text{sparse}}, r, b, N_{\text{refine}}, N_{\text{min}}$.
    \State \textbf{Initialize:} $K \gets 0$, concept set $\mathcal{C} \gets \emptyset$, DSAMs $\mathcal{M} \gets \emptyset$, buffers $\mathcal{B} \gets \emptyset$.
    \State \textbf{Initialize:} Optimizer for trainable params in $\mathcal{M}$; frozen teacher $f_\theta^{\text{frozen}}$.
    \Statex \hfill $\triangleright$ Main training loop
    \ForAll{training steps}
        \State Sample batches $B_e \subset \mathcal{D}_e$, $B_o \subset \mathcal{D}_o$; $B \gets B_e \cup B_o$.
        \ForAll{samples $(I, T, y)$ in $B$}
            \State $\mathbf{h}_v, \mathbf{h}_t \gets \text{ExtractFeatures}(f_\theta, I, T)$;
            \State $\mathbf{h}_f \gets \text{Fuse}(\mathbf{h}_v, \mathbf{h}_t)$.
            \If{sample $\in B_e$}
                \State \Call{UpdateConceptsAndBuffers}{$\mathbf{h}_f, \mathbf{h}_v$}
            \EndIf
            \State $\mathcal{C}_{\text{cand}} \gets \text{FindCandidates}(\mathbf{h}_v, \{\mathbf{p}_{k,v}\})$.
            \State $\{w_k\} \gets \text{Route}(\mathbf{h}_f, \{\mathbf{p}_{k,f}\}_{k \in \mathcal{C}_{\text{cand}}})$.
            \State $\Delta \mathbf{h}_f \gets \sum_{k \in \mathcal{C}_{\text{cand}} \text{ \& } \text{DSAM}_k \text{ active}} w_k \Psi_k(\mathbf{h}_f)$.
            \State $\mathbf{h}'_f \gets \mathbf{h}_f + \Delta \mathbf{h}_f$.
        \EndFor
        \State Compute $\mathcal{L}_{\text{task}}$, $\mathcal{L}_{\text{align}}$ on $B_e$.
        \State Compute $\mathcal{L}_{\text{cdistill}}$, $\mathcal{L}_{\text{sparse}}$ on $B_o$ and $f_\theta^{\text{frozen}}$.
        \State $L \gets \mathcal{L}_{\text{task}} + \lambda_{\text{align}}\mathcal{L}_{\text{align}} + \lambda_{\text{distill}}\mathcal{L}_{\text{cdistill}} + \lambda_{\text{sparse}}\mathcal{L}_{\text{sparse}}$.
        \State Update trainable DSAM parameters using $\nabla L$.
        \If{global\_step \% $N_{\text{refine}} = 0$}
            \State \Call{RefineSubspaces}{}
        \EndIf
    \EndFor
    \end{algorithmic}
    \end{algorithm}
\end{minipage}\hfill
\begin{minipage}[t]{0.5\textwidth}
    \begin{algorithm}[H]
    \caption{Helper Procedures for DSCA}
    \label{alg:procedures}
    \begin{algorithmic}[1]
    \Procedure{UpdateConceptsAndBuffers}{$\mathbf{h}_f, \mathbf{h}_v$}
        \If{$K > 0$}
            \State $j \gets \arg \min_{k=1..K} \|\mathbf{h}_f - \mathbf{p}_{k,f}\|_2$
        \EndIf
        \If{$K=0$ or $\|\mathbf{h}_f - \mathbf{p}_{j,f}\|_2 > \delta_j$} \Comment{Novel concept}
            \State $K \gets K+1$; let new index be $K$.
            \State Initialize $C_K$ with $\mathbf{p}_{K,f} \gets \mathbf{h}_f$, $\mathbf{p}_{K,v} \gets \mathbf{h}_v$, $\delta_K$.
            \State Initialize inactive $\text{DSAM}_K$ and empty buffer for $C_K$.
            \State Add $\mathbf{h}_f$ to the buffer of $C_K$.
        \Else \Comment{Existing concept}
            \State Update $\mathbf{p}_{j,f}$, $\mathbf{p}_{j,v}$, and $\delta_j$ via EMA.
            \State Add $\mathbf{h}_f$ to the buffer of $C_j$.
        \EndIf
    \EndProcedure
    \Statex
    \Procedure{RefineSubspaces}{}
        \ForAll{concepts $C_k$ in $\mathcal{C}$}
            \If{$\text{DSAM}_k$ active}
                \State Update basis $\mathbf{R}_k$ using Incremental PCA on residualized features.
            \ElsIf{$\text{DSAM}_k$ inactive \& buffer size $\ge N_{\text{min}}$}
                \State Compute initial basis $\mathbf{R}_k$ via PCA.
                \State Mark $\text{DSAM}_k$ as active.
            \EndIf
            \State Clear buffer for $C_k$.
        \EndFor
    \EndProcedure
    \Statex
    \end{algorithmic}
    \end{algorithm}
\end{minipage}
\end{figure*}
%--------------------------------------%

\subsection{Gated Residual Intervention in Semantic Subspaces}
\label{sec:routing_and_intervention}

Given routed weights and DSAM proposals, DSCA aggregates concept-wise interventions into a single residual update.

The intervention produced by DSAM $k$ is the gated subspace update
\begin{equation}
\label{eq:intervention_psi}
\Psi_k(\mathbf{h}_f) = \mathbf{\Gamma}_k(\mathbf{h}_f) \left[ \mathbf{R}_k^\top \left( (\mathbf{W}_k \mathbf{h}_f + \mathbf{b}_k) - \mathbf{R}_k \mathbf{h}_f \right) \right].
\end{equation}

The final edited representation is then
\begin{equation}
\label{eq:final_update}
\mathbf{h}'_f = \mathbf{h}_f + \sum_{k=1}^{K} w_k \Psi_k(\mathbf{h}_f).
\end{equation}

For inputs that clearly correspond to a single concept $j$, the routing distribution is typically peaked ($w_j \approx 1$) and the update is dominated by a single DSAM. For ambiguous cases, multiple DSAMs can contribute, allowing DSCA to blend nearby concept subspaces. Under approximately orthogonal $\{R_k\}$, we show in Supplementary that edits in one subspace have provably bounded interference on others.

%--------------------------------------%

\subsection{Multi-Objective Training Objective}
\label{sec:loss}

We train DSCA to jointly optimize task fidelity, locality, and cross-modal alignment using a composite loss over edit samples $\mathcal{D}_e$ and replay samples $\mathcal{D}_o$.

We use four components:
\begin{itemize}
    \item \textbf{Task fidelity loss} $\mathcal{L}_{\text{task}}$:
    a standard causal language modeling loss on edit samples, encouraging $\mathbf{h}'_{f,e}$ to produce the desired target sequence.
    \item \textbf{Cross-modal alignment loss} $\mathcal{L}_{\text{align}}$:
    a cosine-similarity regularizer aligning the edited fused representation $\mathbf{h}'_{f,e}$ with the unmodified text representation $\mathbf{h}_{t,e}$, anchoring edits in the textual semantic space.
    \item \textbf{Contrastive distillation loss} $\mathcal{L}_{\text{cdistill}}$:
    an InfoNCE-style loss\cite{oord2018representation} that encourages each replay representation $\mathbf{h}'_{f,o}$ to remain closest to its frozen-teacher counterpart, preserving the relational geometry of non-edited samples.
    \item \textbf{Gate sparsity loss} $\mathcal{L}_{\text{sparse}}$:
    an $\ell_1$ penalty on routing logits $\{z_k\}$ for replay samples, discouraging spurious activations of DSAMs for out-of-scope inputs.
\end{itemize}

The overall training objective is a weighted sum:
\begin{equation}
    \mathcal{L} = \mathcal{L}_{\text{task}} + \lambda_{\text{align}}\mathcal{L}_{\text{align}} + \lambda_{\text{distill}}\mathcal{L}_{\text{cdistill}} + \lambda_{\text{sparse}}\mathcal{L}_{\text{sparse}},
    \label{eq:total_loss}
\end{equation}
where the $\lambda$s balance plasticity (successful edits) and stability (knowledge retention). A frozen copy of the backbone VLM provides teacher features for $\mathcal{L}_{\text{cdistill}}$, and we interleave edit and replay batches during training.\\
\textbf{Practical update scheme.}
In practice, DSCA separates fast, gradient-driven parameters from slower, data-driven structural components. For each concept $k$, the intervention parameters $\{W_k, b_k, W_{g,k}, b_{g,k}\}$ (Sec.\ref{sec:dsam}) are updated by backpropagation at every step using the composite loss in Eq.\ref{eq:total_loss}, while the concept prototypes $\{p_{k,v}, p_{k,f}\}$(Sec.\ref{sec:concept_partitioning},\ref{sec:HierarchicalGating}) and semantic bases $R_k$ (Sec.\ref{sec:dsam}) are never optimized by gradient descent. Instead, prototypes are updated via exponential moving average whenever a sample is assigned to cluster $C_k$(Sec.\ref{sec:concept_partitioning}), and each subspace $R_k$ is initialized once $N_{\text{min}}$ samples are available and periodically refined using residualized Incremental PCA over buffered features. This dual-mode update scheme turns the subspaces into a slowly evolving ``knowledge base'' on top of which DSAMs can rapidly adapt to new edits. The overall training procedure is summarized in Algorithm~\ref{alg:main_loop}. 
The helper routines are provided in Algorithm~\ref{alg:procedures}

%--------------------------------------%

\section{Experiments}
\label{sec:experiments}

\subsection{Setup}
We implement DSCA on the LLaVA-1.5-7B model \cite{liu2023improvedllava} and evaluate its performance against state-of-the-art editing methods, including LiveEdit \cite{chen2025lifelong}, DualEdit \cite{shi2024dualedit}, MEND \cite{mitchell2022fast}, LTE \cite{jiang2024learning}, VisEdit \cite{chen2025attribution} and SERAC \cite{mitchell2022memory}. To demonstrate architectural generality, we also apply DSCA to the PaliGemma-3B model \cite{beyer2024paligemma} on the CoIN continual learning benchmark \cite{chen2024coin}, following the PAM protocol \cite{sokar2025continual}. All experiments are conducted on 8$\times$A100 GPUs with mixed precision.
All DSCA-specific hyperparameters (subspace rank $r$, minimum samples per concept $N_{\text{min}}$, refinement interval $N_{\text{refine}}$, routing temperature $\tau$ and loss weights $\lambda_{\text{align}}, \lambda_{\text{distill}}, \lambda_{\text{sparse}}$) are fixed across experiments and summarized in the Supplementary. Unless otherwise noted, higher values indicate better performance for all metrics (less negative BWT corresponds to less forgetting).

\subsection{Evaluation Metrics}
\label{sec:metrics}

For editing benchmarks (E-VQA, E-IC, VLKEB\cite{huang2024vlkeb}) we follow prior work \cite{cheng2023edit,mitchell2022memory} and report \textbf{Reliability (Rel.)}, \textbf{Textual Generalization (T-Gen.)}, \textbf{Visual/Modal Generalization (V-Gen./M-Gen.)}, \textbf{Textual Locality (T-Loc.)}, \textbf{Multimodal Locality (M-Loc.)}, and their mean \textbf{Average (Avg.)}.

For the CoIN benchmark \cite{chen2024coin}, we follow standard continual-learning practice \cite{10.5555/3295222.3295393} and report \textbf{Average Accuracy (ACC)}, \textbf{Backward Transfer (BWT}, \textbf{Forward Transfer (FWT)}, and \textbf{Average Task Accuracy} $A_t$ (plasticity). Formal metric definitions and equations are given in the Supplementary.
\subsection{Core Editing Efficacy}
We first assess single-edit performance (Table~\ref{tab:combined_single_edit}). Across both E-VQA and E-IC benchmarks \cite{cheng2023edit}, DSCA sets a new state-of-the-art, it improves Avg.\ score from 97.84 to 98.50 on E-VQA and from 97.85 to 98.00 on E-IC relative to the strongest baseline, DualEdit \cite{shi2024dualedit}, while maintaining perfect or near-perfect locality.
\begin{table}[htbp]
\centering
\caption{Single-edit results on E-VQA \cite{cheng2023edit} and E-IC \cite{cheng2023edit} (LLaVA-1.5-7B \cite{liu2023improvedllava}). Baselines from DualEdit~\cite{shi2024dualedit}.}
\label{tab:combined_single_edit}
\resizebox{\columnwidth}{!}{%
\begin{tabular}{llcccccc}
\toprule
Dataset & Method & Rel. & T-Gen. & V-Gen. & T-Loc. & M-Loc. & Avg. \\
\midrule
\multirow{5}{*}{E-VQA} & DualEdit\cite{shi2024dualedit} & 96.94 & 96.43 & 96.20 & 100.00 & 99.61 & 97.84 \\
& VisEdit\cite{chen2025attribution} & 95.78 & 94.21 & 94.37 & 100.00 & 91.11 & 95.09 \\
& SERAC \cite{mitchell2022memory} & 82.51 & 81.60 & 80.05 & 100.00 & 57.48 & 80.33 \\
& LTE \cite{jiang2024learning} & 94.16 & 93.54 & 93.06 & 83.76 & 81.65 & 89.23 \\
& MEND\cite{mitchell2022fast} & 92.30 & 92.16 & 92.10 & 90.30 & 81.13 & 89.60 \\
& \textbf{DSCA (ours)} & \textbf{98.12} & \textbf{97.30} & \textbf{97.25} & \textbf{100.00} & \textbf{99.83} & \textbf{98.50} \\
\midrule
\multirow{5}{*}{E-IC} & DualEdit\cite{shi2024dualedit} & 96.76 & 96.52 & 96.24 & 100.00 & 99.74 & 97.85 \\
& VisEdit\cite{chen2025attribution} & 95.06 & 94.87 & 94.35 & 100.00 & 95.23 & 95.90 \\
& SERAC \cite{mitchell2022memory} & 43.08 & 42.37 & 42.85 & 100.00 & 7.63 & 47.19 \\
& LTE \cite{jiang2024learning} & 93.60 & 92.38 & 91.18 & 85.54 & 88.49 & 90.24 \\
& MEND\cite{mitchell2022fast} & 93.76 & 93.46 & 92.14 & 91.60 & 87.59 & 91.71 \\
& \textbf{DSCA (ours)} & \textbf{98.00} & \textbf{97.10} & \textbf{97.02} & \textbf{100.00} & \textbf{99.90} & \textbf{98.00} \\
\bottomrule
\end{tabular}%
}
\end{table}
\subsection{Robustness in a Lifelong Learning Scenario}
We next escalate to lifelong editing, evaluating over $t = 1{,}000$ sequential edits. As shown in Table~\ref{tab:lifelong}, DSCA surpasses LiveEdit \cite{chen2025lifelong} and other baselines on both E-VQA and VLKEB \cite{huang2024vlkeb}. While LiveEdit maintains strong performance, it still exhibits noticeable erosion in reliability (92.93\%) and multimodal locality (96.43\%) on E-VQA after 1{,}000 edits.
DSCA, by contrast, maintains higher reliability (96.85\%) and near-perfect locality (98.20\%) on E-VQA, and achieves similar gains on VLKEB. This suggests that beyond retrieval-based isolation, DSCA's use of approximately orthogonal concept subspaces provides a more principled mechanism for preventing subtle, compounding interference over long edit sequences.
\begin{table}[htbp]
\centering
\caption{Lifelong editing results ($t=1000$ edits) on LLaVA-1.5-7B \cite{liu2023improvedllava}. Baselines from LiveEdit~\cite{chen2025lifelong}.}
\label{tab:lifelong}
\resizebox{\columnwidth}{!}{%
\begin{tabular}{llcccccc}
\toprule
Dataset & Method & Rel. & T-Gen. & M-Gen. & T-Loc. & M-Loc. & Avg. \\
\midrule
\multirow{4}{*}{E-VQA\cite{cheng2023edit}} & LiveEdit\cite{chen2025lifelong} & 92.93 & 90.16 & 84.30 & 100.00 & 96.43 & 92.76 \\
& LTE \cite{jiang2024learning} & 83.93 & 82.55 & 81.34 & 83.97 & 73.09 & 80.98 \\
& MEND\cite{mitchell2022fast} & 0.04 & 0.05 & 0.05 & 0.08 & 0.09 & 0.06 \\
& SERAC \cite{mitchell2022memory} & 85.57 & 75.58 & 82.01 & 62.46 & 15.69 & 64.26 \\
& \textbf{DSCA (ours)} & \textbf{96.85} & \textbf{93.10} & \textbf{88.00} & \textbf{100.00} & \textbf{98.20} & \textbf{95.23} \\
\cmidrule{2-8}
\multirow{4}{*}{VLKEB\cite{huang2024vlkeb}} & LiveEdit\cite{chen2025lifelong} & 92.22 & 83.97 & 82.75 & 100.00 & 100.00 & 91.79 \\
& LTE \cite{jiang2024learning} & 64.51 & 56.26 & 64.80 & 80.85 & 76.52 & 68.59 \\
& MEND\cite{mitchell2022fast} & 0.03 & 0.05 & 0.07 & 0.06 & 0.08 & 0.06 \\
& SERAC \cite{mitchell2022memory} & 60.93 & 56.49 & 60.06 & 52.94 & 15.04 & 49.09 \\
& \textbf{DSCA (ours)} & \textbf{98.10} & \textbf{93.80} & \textbf{89.70} & \textbf{100.00} & \textbf{100.00} & \textbf{96.72} \\
\bottomrule
\end{tabular}
}
\end{table}

\subsection{Continual Learning on CoIN}
This stability is further confirmed on the CoIN benchmark\cite{chen2024coin} (Table~\ref{tab:PAM}). DSCA achieves a Backward Transfer (BWT) of \textbf{-9.37}, indicating minimal forgetting, compared to standard fine-tuning ($-39.51$) and PAM ($-19.45$) \cite{sokar2025continual}. More importantly, this gain in stability is obtained without sacrificing plasticity ($A_t = 76.48$), showing that DSCA achieves a favorable stability--plasticity trade-off.
\begin{table}[htbp]
    \centering
    \caption{Performance on the CoIN benchmark across four metrics (mean$\pm$std over three task orders). Higher is better for all metrics; less negative BWT indicates less forgetting. Baselines from PAM\cite{sokar2025continual}}
    \label{tab:PAM}
\resizebox{\columnwidth}{!}{%
    \begin{tabular}{lcccc}
        \toprule
        \textbf{METHOD} & \textbf{ACC} & \textbf{BWT} & \textbf{FWT} & \textbf{$A_t$} \\
        \midrule
        ZERO-SHOT   & 24.74          & -              & -              & -              \\
        INDEPENDENT & 76.46          & -              & -              & -              \\
        MULTITASK   & 73.93          & -              & -              & -              \\
        \midrule
        MoELoRA \cite{chen2024coin}         & 46.59$\pm$9.98 & -36.40$\pm$11.97& 7.79$\pm$2.24  & \textbf{76.93$\pm$0.27} \\
        MagMax\cite{marczak2024magmax}      & 45.74$\pm$0.88 & -22.68$\pm$6.51 & 4.75$\pm$3.36  & 76.29$\pm$0.18 \\
        PAM\cite{sokar2025continual}        & 49.89$\pm$1.66 & -19.45$\pm$0.95 & \textbf{11.11$\pm$0.09} & 76.31$\pm$0.03 \\
        DSCA (ours, PaliGemma-3B)           & \textbf{49.96$\pm$0.72} & \textbf{-9.37$\pm$1.02} & 11.04$\pm$0.13 & 76.48$\pm$0.07 \\
        \bottomrule
    \end{tabular}
    }
\end{table}

\subsection{Safeguarding Foundational VLM Capabilities}

A practical editor must remain benign with respect to the base model's general capabilities. We therefore measure post-edit performance on standard LVLM benchmarks MME \cite{fu2023mme}, MM-Vet \cite{yu2023mmvet}, VQA-v2 \cite{balanced_vqa_v2}, TextVQA \cite{singh2019towards}, and COCO Captions \cite{10.1007/978-3-319-10602-1_48} (CIDEr \cite{vedantam2015cider}). As shown in Table~\ref{tab:retention}, DSCA matches or exceeds strong baselines such as LiveEdit\cite{chen2025lifelong} and DualEdit on all benchmarks, indicating that high locality in the representation space effectively insulates foundational knowledge from degradation.
We also examine whether DSCA mitigates common LVLM failure modes such as object hallucination. Using the CHAIR metric \cite{rohrbach-etal-2018-object}, where lower scores are better, Table~\ref{tab:hallucination} shows that DSCA significantly reduces hallucination rates compared to prior editors, achieving a CHAIR-H score of \textbf{15.9} vs.\ 21.1 for LiveEdit and 20.8 for DualEdit\cite{shi2024dualedit}, improving over the previous state-of-the-art Gen-Anchor Rep-Edit \cite{shi2025exposing}. We attribute this to DSCA's constrained, approximately orthogonal semantic subspaces, which avoid activating loosely related concepts that trigger hallucinated objects.
\begin{table}[htbp]
\centering
\caption{Capability retention on standard LVLM benchmarks.}
\label{tab:retention}
\resizebox{\columnwidth}{!}{%
\begin{tabular}{lccccc}
\toprule
Method & MME\cite{fu2023mme} & MM-Vet\cite{yu2023mmvet} & VQA-v2\cite{balanced_vqa_v2} & TextVQA \cite{singh2019towards} & COCO CIDEr\cite{10.1007/978-3-319-10602-1_48}\cite{vedantam2015cider} \\
\midrule
LiveEdit\cite{chen2025lifelong}   & 73.4 & 54.2 & 82.1 & 61.5 & 120.6 \\
DualEdit\cite{shi2024dualedit}    & 74.1 & 55.6 & 83.5 & 62.9 & 121.4 \\
MRT \cite{751973dae1e349d19373b7f59f30e536}                               & 70.9 & 50.8 & 80.2 & 59.3 & 118.2 \\
\textbf{DSCA (ours)}               & \textbf{76.3} & \textbf{57.8} & \textbf{84.9} & \textbf{64.7} & \textbf{123.8} \\
\bottomrule
\end{tabular}
}
\end{table}
\begin{table}[htbp]
\centering
\caption{Hallucination stress test (CHAIR) \cite{rohrbach-etal-2018-object}. Lower CHAIR-H is better.}
\label{tab:hallucination}
\resizebox{\columnwidth}{!}{%
\begin{tabular}{lccc}
\toprule
Method & CHAIR-H $\downarrow$ & Faithfulness $\uparrow$ & Richness $\approx$ \\
\midrule
Gen-Anchor Rep-Edit\cite{shi2025exposing} & 18.5 & 92.3 & 98.1 \\
DualEdit\cite{shi2024dualedit}            & 20.8 & 90.6 & 97.2 \\
LiveEdit\cite{chen2025lifelong}           & 21.1 & 91.0 & 96.8 \\
\textbf{DSCA (ours)}                      & \textbf{15.9} & \textbf{94.5} & \textbf{98.4} \\
\bottomrule
\end{tabular}
}
\end{table}

\subsection{Diagnostic Analysis and Ablation Studies}
\label{sec:diagnostics}

\begin{figure*}[htbp]
    \centering
    \includegraphics[width=\textwidth]{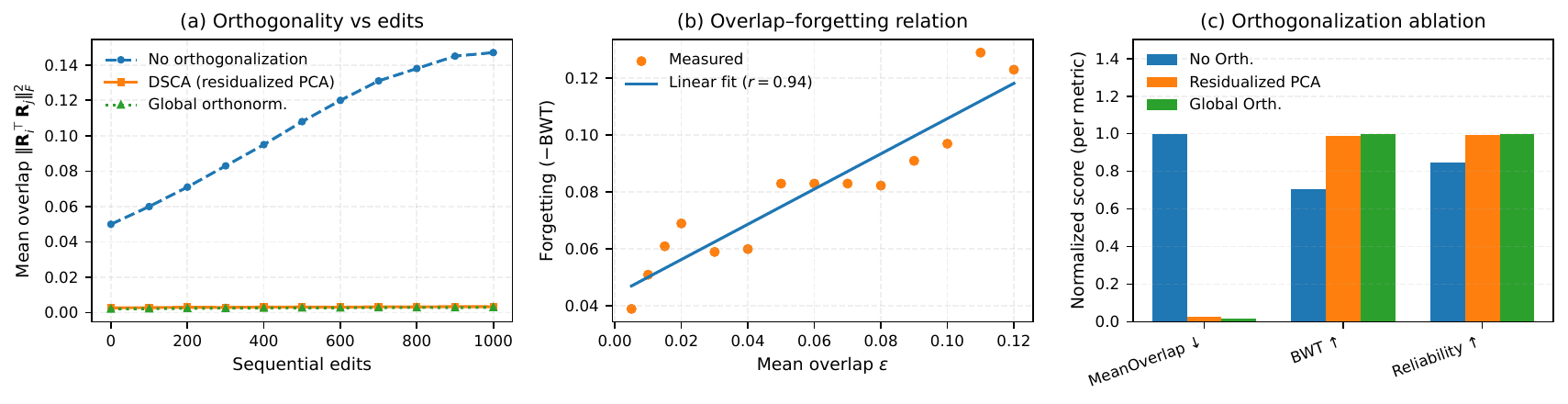}
    \vspace{-2mm}
    \caption{\textbf{Diagnostic analysis of DSCA.}
    \textbf{(a)} Mean pairwise subspace overlap($\varepsilon=\|\mathbf{R}_i^\top \mathbf{R}_j\|_F^2$) as a function of the number of sequential edits. 
    DSCA with residualized Incremental PCA keeps the overlap essentially flat at $\approx 3\times 10^{-3}$ across $1{,}000$ edits, comparable to a globally orthonormal baseline, whereas a variant without orthogonalization drifts to more than $10^{-1}$.
    \textbf{(b)} Relationship between mean subspace overlap and forgetting, measured as $1-\text{BWT}$. 
    The scatter plot and fitted line (Pearson $r\!\approx\!0.94$) show a strong monotonic dependence, as overlap $\varepsilon$ increases, forgetting grows almost linearly, consistent with the bounded-interference result in %Corollary~\ref{cor:bounded_interference}.
    Supplementary.
    \textbf{(c)} Ablation over subspace construction strategies (No orthogonalization, residualized PCA, global orthonormalization). 
    For each metric (mean overlap, BWT, reliability) scores are normalized per metric to $[0,1]$ for comparability. 
    Residualized PCA attains nearly the same BWT and reliability as global orthonormalization while dramatically reducing overlap relative to the non-orthogonal baseline, giving the best overall trade-off.}
    \label{fig:dsca_diagnostics}
    \vspace{-3mm}
\end{figure*}

We empirically validate the geometric properties claimed in Sec.~\ref{sec:dsam} in Supplementary. 
Fig.~\ref{fig:dsca_diagnostics}(a) shows that DSCA keeps the mean pairwise subspace overlap $\|\mathbf{R}_i^\top \mathbf{R}_j\|_F^2$ essentially constant at $\approx 3\times 10^{-3}$ over $1{,}000$ edits, close to the globally orthonormal baseline, whereas a variant that omits residualized orthogonalization drifts to overlaps above $10^{-1}$. 
Fig.~\ref{fig:dsca_diagnostics}(b) plots forgetting (measured as $1-\text{BWT}$) against the mean overlap $\varepsilon$ and reveals an almost linear trend with a high Pearson correlation ($r\!\approx\!0.94$), providing empirical support for the bounded-interference Corollary in the
Supplementary.
%Corollary~\ref{cor:bounded_interference}
As subspaces become less orthogonal, forgetting increases in a predictable way. 
Fig.\ref{fig:dsca_diagnostics}(c) summarizes the ablation across three metrics mean overlap, BWT, and edit reliability after normalizing each metric across methods. 
Residualized PCA achieves nearly the same BWT and reliability as global orthonormalization while strongly reducing overlap compared to the non-orthogonal baseline, which justifies our choice of residualized PCA as the default subspace construction in DSCA.

To dissect the contribution of each component in DSCA, we perform an ablation study summarized in Table~\ref{tab:ablation_component}. We report Edit Success (ES; higher is better), cumulative Locality Drop (Locality $\Delta$; lower is better), and Generalization (GEN; higher is better). The full DSCA model attains ES of 98.0, Locality $\Delta$ of only 0.5, and GEN of 97.3.

The central role of orthogonality is evident, removing it (\emph{w/o orthogonality}) increases the Locality Drop by more than $5\times$ (0.5 to 2.8) and significantly harms GEN. Gate sparsity is also critical, setting $\lambda_{\text{sparse}} = 0$ (\emph{w/o gate sparsity}) leads to dense module activation, raising Locality $\Delta$ to 2.1 and reducing ES to 96.1. Simplifying the hierarchical routing to a single stage or removing the basis-residual update similarly degrades locality, confirming that minimal, targeted residuals in concept-specific subspaces are key. Finally, reducing the number of subspaces (K/2) or the rank (r/2) produces predictable, modest drops, indicating that DSCA is robust to moderate capacity reductions. The effectiveness of our sparse design is visualized in Figure~\ref{fig:routsparse}. The histogram in Figure~\ref{fig:routsparse}(a) reveals a highly sparse activation pattern in the full model, with over 95\% of routing weights being negligible (near zero). This sparsity is by design, as shown in Figure~\ref{fig:routsparse}(b), which illustrates the trade-off between the sparsity loss coefficient and the number of active modules. Our chosen operating point (the blue dot) achieves a highly efficient state where, on average, only three DSAMs are activated per input.
Additional qualitative analyses, including concept-wise t-SNE visualizations of projected representations are provided in the Supplementary.

\begin{figure}[t]
    \centering
    \includegraphics[width=\columnwidth]{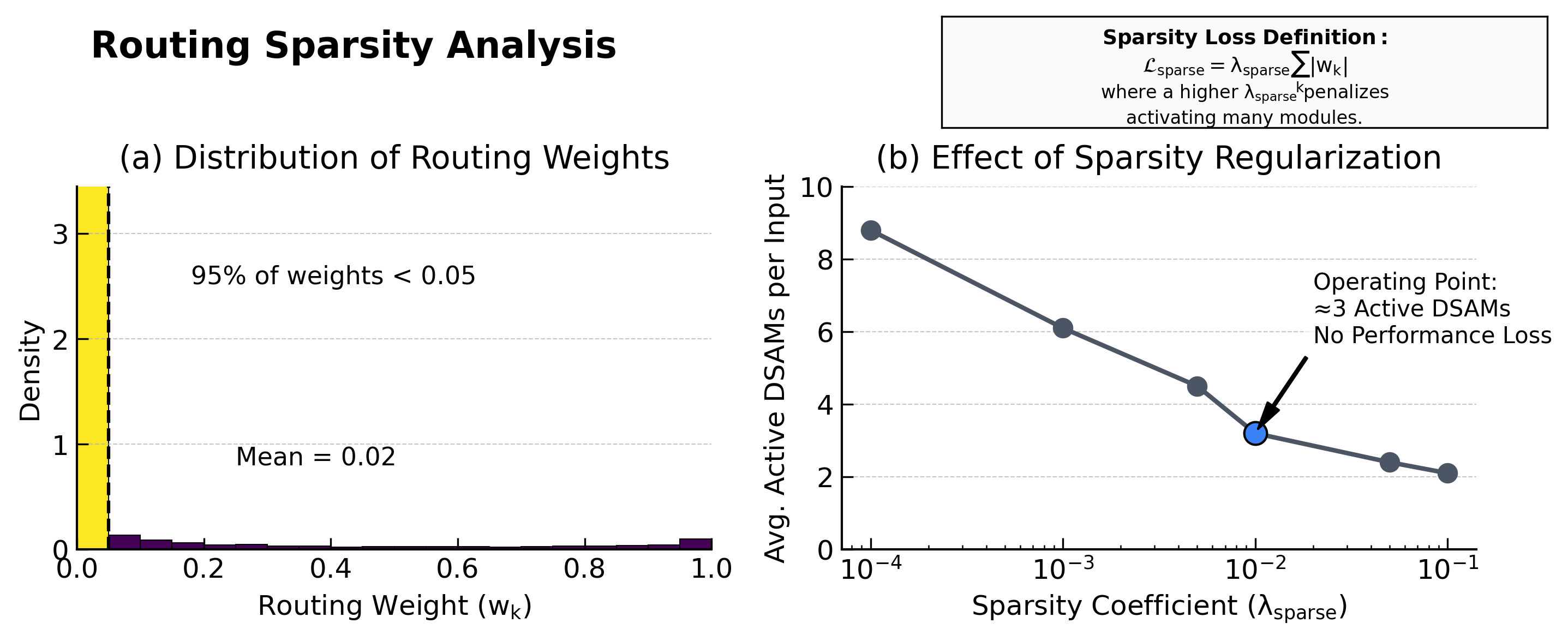}
    \vspace{-2mm}
    \caption{\textbf{Routing sparsity analysis.}
    (a) Histogram of routing weights shows that over 95\% are below 0.05, indicating highly selective module activation.
    (b) Trade-off between sparsity coefficient $\lambda_{\text{sparse}}$ and average number of active DSAMs. The chosen operating point (blue dot) yields $\approx 3$ active DSAMs per input without degrading performance.}
    \label{fig:routsparse}
    \vspace{-3mm}
\end{figure}
\begin{table*}[htbp]
\centering
\caption{Ablation study on DSCA components. A checkmark ($\checkmark$) indicates the component is present; a cross ($\times$) indicates it has been ablated or modified.}
\label{tab:ablation_component}

\newcolumntype{C}{>{\centering\arraybackslash}m{1.6cm}}

\begingroup
\renewcommand{\arraystretch}{0.95}
\resizebox{\linewidth}{!}{%
\begin{tabular}{l *{6}{C} | ccc}
\toprule
\textbf{Variant} &
\shortstack[c]{\strut\textbf{Orthogonality}\\\textbf{in $R_k$}\strut} &
\shortstack[c]{\strut\textbf{Gate}\\\textbf{Sparsity}\strut} &
\shortstack[c]{\strut\textbf{Multi-stage}\\\textbf{Routing}\strut} &
\shortstack[c]{\strut\textbf{Basis Residual}\\\textbf{intervention}\strut} &
\shortstack[c]{\strut\textbf{Full}\\\textbf{Subspace}\strut} &
\shortstack[c]{\strut\textbf{Full}\\\textbf{Rank}\strut} &
\textbf{ES} $\uparrow$ &
\shortstack[c]{\strut\textbf{Locality}\\$\Delta \downarrow$\strut} &
\textbf{GEN} $\uparrow$
\\
\midrule
Full DSCA & \cmark & \cmark & \cmark & \cmark & \cmark & \cmark & 98.0 & 0.5 & 97.3 \\
\midrule
w/o orthogonality     & \xmark & \cmark & \cmark & \cmark & \cmark & \cmark & 95.8 & 2.8 & 93.4 \\
w/o gate sparsity     & \cmark & \xmark & \cmark & \cmark & \cmark & \cmark & 96.1 & 2.1 & 94.7 \\
single-stage routing  & \cmark & \cmark & \xmark & \cmark & \cmark & \cmark & 96.9 & 1.9 & 95.0 \\
no basis-residual     & \cmark & \cmark & \cmark & \xmark & \cmark & \cmark & 97.1 & 1.5 & 95.8 \\
fewer subspaces (K/2) & \cmark & \cmark & \cmark & \cmark & \xmark & \cmark & 97.0 & 1.7 & 95.5 \\
lower rank (r/2)      & \cmark & \cmark & \cmark & \cmark & \cmark & \xmark & 97.3 & 1.2 & 96.1 \\
\bottomrule
\end{tabular}%
}
\endgroup
\end{table*}

\subsection{Discussion}

Across all evaluations, DSCA achieves strong edit success, robust generalization, and near-perfect locality in both single-edit and challenging lifelong scenarios. By routing inputs into sparsely activated, approximately orthogonal concept subspaces and applying gated residual interventions within them, DSCA turns forgetting into a controlled geometric quantity rather than an emergent side effect of optimization. This geometry-aware design yields favorable stability--plasticity trade-offs on CoIN while preserving broad LVLM capabilities, making DSCA a practical tool for maintaining and refining large VLMs under continual editing.

\section{Conclusion and Limitations}
We introduced DSCA, a Dynamic Subspace Concept Alignment framework for editing large vision--language models. DSCA partitions the fused representation space into concept-specific semantic subspaces and attaches lightweight Dynamic Structured Alignment Modules that operate within these subspaces. Combined with sparse two-stage routing and a multi-objective loss over task fidelity, locality, and cross-modal alignment, this architecture enables precise, minimally invasive edits. Empirically, DSCA achieves state-of-the-art performance on single-edit and lifelong editing benchmarks, while retaining the core capabilities of the underlying LVLM and substantially reducing catastrophic forgetting in continual learning settings. On CoIN, DSCA attains a markedly better stability--plasticity trade-off, indicating that controlling interference via (approximately) orthogonal semantic subspaces is an effective design principle.

DSCA still has several limitations. First, it relies on a linear semantic subspace model, each concept is represented by a low-rank basis $R_k$ and edits are implemented as gated linear residuals, which may be restrictive for highly non-linear or entangled concepts. Second, maintaining approximately orthogonal subspaces becomes more costly as the number of concepts $K$ grows, suggesting the need for additional compression, sharing, or sparsity in the subspace representations. Third, DSCA depends on reliable concept discovery and routing; when concepts are highly overlapping or ambiguous, misassignments can lead to suboptimal edits or residual interference. Extending DSCA to richer non-linear subspaces or hyper-network parameterizations, exploring tighter integration with backbone model training, and applying the framework to other modalities (e.g., video--language, audio--visual, or embodied agents) are promising directions for future work.

\section{Acknowledgments}
This work was supported by Zynix AI's Foundational Research Grant, whose support enabled the exploration and development of the ideas presented in this work. We gratefully acknowledge Zynix AI for providing the research environment, computational infrastructure, and collaborative ecosystem that made this research possible.

We are especially grateful to Gautamdev Chowdary, CTO and Dr. Jayadeva Chowdappa, CEO for their encouragement, insightful discussions, and continued support of foundational AI research. Their perspective and feedback helped shape the direction of this work.
We also thank the broader Zynix AI team for valuable technical discussions and feedback throughout the course of this research.

{
    \small
    \bibliographystyle{ieeenat_fullname}
    \bibliography{main}
}

% WARNING: do not forget to delete the supplementary pages from your submission 
\clearpage
\setcounter{page}{1}
\maketitlesupplementary

\section{Contents}
{\color{blue}
\begin{itemize}
    \item Theoretical Analysis of Non-Interference in DSCA
    \item Additional Methodology Details
    \item Evaluation Metrics
    \item Implementation Details and Hyperparameters
    \item Extended Experimental Results
\end{itemize}
}

% ================== THEORY ================== %
\section{Theoretical Analysis of Non-Interference in DSCA}
\label{app:theory}

\subsection{Preliminaries}

Let the frozen VLM encoder produce fused representations
$\mathbf{h}_f \in \mathbb{R}^{d_f}$(as defined in Sec. \ref{sec:problem_formulation}).
For each discovered concept $C_k$, DSCA maintains a low-dimensional
semantic subspace with basis matrix
$\mathbf{R}_k \in \mathbb{R}^{r_k \times d_f}$, where $r_k \ll d_f$.(Sec.~\ref{sec:dsam})
We view the rows of $\mathbf{R}_k$ as an orthonormal basis for the
concept subspace
\begin{equation}
    \mathcal{R}_k = \left\{ \mathbf{R}_k^\top \mathbf{u} \; \middle| \;
    \mathbf{u} \in \mathbb{R}^{r_k} \right\}
    \subset \mathbb{R}^{d_f}.
\end{equation}
The corresponding orthogonal projector onto $\mathcal{R}_k$ is
\begin{equation}
    \mathbf{P}_k = \mathbf{R}_k^\top \mathbf{R}_k \in \mathbb{R}^{d_f \times d_f}.
\end{equation}

In Sec.~\ref{sec:dsam}, the intervention proposed by DSAM $k$ is
\begin{equation}
\label{eq:app_dsam_core}
    \Psi_k(\mathbf{h}_f)
    = \mathbf{\Gamma}_k(\mathbf{h}_f) \,
      \mathbf{R}_k^\top
      \big[
        (\mathbf{W}_k \mathbf{h}_f + \mathbf{b}_k)
        - \mathbf{R}_k \mathbf{h}_f
      \big],
\end{equation}
and the full edited representation (Sec.~\ref{sec:routing_and_intervention})
is
\begin{equation}
\label{eq:app_full_update}
    \mathbf{h}'_f
    = \mathbf{h}_f + \sum_{k=1}^{K} w_k \,\Psi_k(\mathbf{h}_f),
\end{equation}
where $\{w_k\}$ are routing weights (Sec.~\ref{sec:HierarchicalGating})
and $\mathbf{\Gamma}_k(\mathbf{h}_f)$ is the diagonal gating matrix
defined in Eq.~\ref{eq:diag}.

For the theoretical analysis, it is convenient to absorb the
component-wise gate into an effective subspace update.  We therefore
rewrite
\begin{equation}
    \Psi_k(\mathbf{h}_f)
    = \mathbf{R}_k^\top \,\Delta_k(\mathbf{h}_f),
\end{equation}
where $\Delta_k(\mathbf{h}_f) \in \mathbb{R}^{r_k}$ is an
\emph{effective} low-dimensional update that depends on
$\mathbf{W}_k, \mathbf{b}_k$ and $\mathbf{\Gamma}_k(\mathbf{h}_f)$.
This does not change the functional form of DSAM in practice, but lets
us express the update as a sum of subspace-aligned residuals.

\subsection{Assumptions}
\label{sec:assum}

We now make explicit the structural assumptions under which non-interference
guarantees hold.

\paragraph{A1. Orthogonal subspaces.}
The semantic subspaces are mutually orthogonal:
\begin{equation}
    \mathbf{R}_i \mathbf{R}_j^\top = \mathbf{0}
    \quad \text{for all } i \neq j,
\end{equation}
and each basis is row-orthonormal,
$\mathbf{R}_k \mathbf{R}_k^\top = \mathbf{I}_{r_k}$.
This implies that the projectors satisfy
$\mathbf{P}_i \mathbf{P}_j = \mathbf{0}$ for $i \neq j$.

\paragraph{A2. Bounded routing and gating.}
For any $\mathbf{h}_f$, the routed update has bounded energy:
\begin{equation}
    \sum_{k=1}^{K} |w_k| \, \|\Delta_k(\mathbf{h}_f)\|_2
    \;\leq\; \Gamma_{\max},
\end{equation}
for some constant $\Gamma_{\max} > 0$.

\paragraph{A3. Concept-aligned training.}
Each DSAM $k$ is trained only on samples assigned to concept $C_k$
by the online clustering as in Sec.~\ref{sec:concept_partitioning} .
DSAM $k$ learns to operate only within its own semantic
region.

In practice, Assumption~A1 is only approximately satisfied.  The
residualized Incremental PCA procedure described in
Sec.~\ref{sec:dsam} and detailed further in and detailed further in Sec.\ref{sec:dual_mode}, maintains small
overlap between subspaces, which we quantify in
Sec.~\ref{sec:orthogonality_diag}.

\subsection{Non-Interference Under Orthogonal Subspaces}

We first state a clean result under exact orthogonality, then extend
to the approximate case.

\begin{lemma}[Non-interference for orthogonal subspaces]
\label{lem:non_interference}
Let $\mathbf{h}_f \in \mathbb{R}^{d_f}$ and
$\mathbf{h}'_f$ be related by Eq.~\ref{eq:app_full_update}.
Suppose Assumption~A1(in Sec.\ref{sec:assum}) holds.
Then, for any concept index $i$,
\begin{equation}
    \mathbf{P}_i \mathbf{h}'_f
    = \mathbf{P}_i \mathbf{h}_f
      + w_i \,\mathbf{P}_i \Psi_i(\mathbf{h}_f).
\end{equation}
In particular, for all $j \neq i$,
the contribution of $\Psi_j$ vanishes in the projection onto
$\mathcal{R}_i$:
\begin{equation}
    \mathbf{P}_i \Psi_j(\mathbf{h}_f) = \mathbf{0}
    \quad \text{for } j \neq i.
\end{equation}
\end{lemma}

\begin{proof}
Using the effective subspace form
$\Psi_k(\mathbf{h}_f) = \mathbf{R}_k^\top \Delta_k(\mathbf{h}_f)$ and
Eq.~\ref{eq:app_full_update},
\begin{equation}
    \mathbf{h}'_f
    = \mathbf{h}_f + \sum_{k=1}^{K}
      w_k \,\mathbf{R}_k^\top \Delta_k(\mathbf{h}_f).
\end{equation}
Projecting onto $\mathcal{R}_i$ gives
\begin{align}
    \mathbf{P}_i \mathbf{h}'_f
    &= \mathbf{P}_i \mathbf{h}_f
       + \sum_{k=1}^{K}
         w_k \,\mathbf{P}_i \mathbf{R}_k^\top \Delta_k(\mathbf{h}_f) \\
    &= \mathbf{P}_i \mathbf{h}_f
       + \sum_{k=1}^{K}
         w_k \,\mathbf{R}_i^\top
         \big( \mathbf{R}_i \mathbf{R}_k^\top \big)
         \Delta_k(\mathbf{h}_f).
\end{align}
By Assumption~A1(in Sec.\ref{sec:assum}), $\mathbf{R}_i \mathbf{R}_k^\top = \mathbf{0}$ for
all $k \neq i$, and $\mathbf{R}_i \mathbf{R}_i^\top = \mathbf{I}$.
Therefore all cross-concept terms vanish and only the $k=i$ term
remains:
\begin{equation}
    \mathbf{P}_i \mathbf{h}'_f
    = \mathbf{P}_i \mathbf{h}_f
      + w_i \,\mathbf{R}_i^\top
      \Delta_i(\mathbf{h}_f)
    = \mathbf{P}_i \mathbf{h}_f
      + w_i \,\mathbf{P}_i \Psi_i(\mathbf{h}_f),
\end{equation}
which proves both claims.
\end{proof}

Lemma~\ref{lem:non_interference} shows that, under ideal orthogonality,
\emph{edits in one semantic subspace cannot alter projections onto any
other subspace}.  In this sense, catastrophic interference is ruled
out at the level of the subspace decomposition.

\subsection{Approximate Orthogonality and Bounded Interference}

In practice, residualized Incremental PCA produces subspaces that are
only approximately orthogonal.  We therefore consider the more
realistic setting where inter-subspace overlap is small but non-zero.

Define the pairwise subspace overlap as
\begin{equation}
    \mathrm{Overlap}(i,j)
    = \big\| \mathbf{R}_i \mathbf{R}_j^\top \big\|_F^2,
    \quad i \neq j,
\end{equation}
and assume a uniform bound
\begin{equation}
\label{eq:eps_overlap}
    \big\| \mathbf{R}_i \mathbf{R}_j^\top \big\|_F^2
    \;\leq\; \varepsilon
    \quad \text{for all } i \neq j.
\end{equation}

\begin{corollary}[Bounded interference under $\varepsilon$-overlap]
\label{cor:bounded_interference}
Suppose Assumptions~A2(in Sec.\ref{sec:assum}) and~\eqref{eq:eps_overlap} hold.
Let $\mathbf{h}'_f$ be given by Eq.~\ref{eq:app_full_update}.  Then,
for any concept index $i$,
the interference from all other concepts in the projection onto
$\mathcal{R}_i$ is bounded by
\begin{equation}
    \Big\|
      \mathbf{P}_i (\mathbf{h}'_f - \mathbf{h}_f)
      - w_i \mathbf{P}_i \Psi_i(\mathbf{h}_f)
    \Big\|_2
    \;\leq\;
    \Gamma_{\max} \,\sqrt{\varepsilon}.
\end{equation}
\end{corollary}

\begin{proof}
From the same expansion as in Lemma~\ref{lem:non_interference}, the
net contribution of all \emph{other} concepts to the projection onto
$\mathcal{R}_i$ is
\begin{equation}
    \mathbf{r}_i
    = \sum_{j \neq i}
      w_j \,\mathbf{P}_i \Psi_j(\mathbf{h}_f)
    = \sum_{j \neq i}
      w_j \,\mathbf{R}_i^\top
      \big( \mathbf{R}_i \mathbf{R}_j^\top \big)
      \Delta_j(\mathbf{h}_f).
\end{equation}
Using the triangle inequality and Cauchy--Schwarz,
\begin{align}
    \|\mathbf{r}_i\|_2
    &\leq
    \sum_{j \neq i}
    |w_j| \,
    \big\|
        \mathbf{R}_i^\top
        \big( \mathbf{R}_i \mathbf{R}_j^\top \big)
        \Delta_j(\mathbf{h}_f)
    \big\|_2 \\
    &\leq
    \sum_{j \neq i}
    |w_j| \,
    \big\|
        \mathbf{R}_i \mathbf{R}_j^\top
    \big\|_F \,
    \big\|\Delta_j(\mathbf{h}_f)\big\|_2 \\
    &\leq
    \sqrt{\varepsilon}
    \sum_{j \neq i}
    |w_j| \,\big\|\Delta_j(\mathbf{h}_f)\big\|_2.
\end{align}
By Assumption~A2(in Sec.\ref{sec:assum}),
$\sum_{j} |w_j| \,\|\Delta_j(\mathbf{h}_f)\|_2 \leq \Gamma_{\max}$,
which yields the desired bound.
\end{proof}

Corollary~\ref{cor:bounded_interference} implies that interference
between concepts scales only with $\sqrt{\varepsilon}$, the square
root of the maximum inter-subspace overlap.  When subspaces are nearly
orthogonal (small $\varepsilon$), the effect of other concepts on the
projection of $\mathbf{h}_f$ onto $\mathcal{R}_i$ is provably small.
This confirms the findings visualized in Fig.\ref{fig:dsca_diagnostics}(a)

\subsection{Empirical Orthogonality Diagnostics}
\label{sec:orthogonality_diag}

To verify that the assumptions above hold approximately in practice, we
monitor the empirical subspace overlap throughout training.  For a
given checkpoint, we compute
\begin{equation}
    \mathrm{MeanOverlap}
    = \frac{1}{K(K-1)} \sum_{i \neq j}
      \big\| \mathbf{R}_i \mathbf{R}_j^\top \big\|_F^2.
\end{equation}
Here, $K$ represents the total count of concept subspaces currently instantiated by the model, corresponding to the dynamic cluster set size defined in Sec.\ref{sec:concept_partitioning}.
In our experiments, the residualized Incremental PCA procedure keeps
$\mathrm{MeanOverlap}$ in the range of $10^{-3}$, consistent with a
small $\varepsilon$ regime.  Combined with the bounded-update
assumption (A2,in Sec.\ref{sec:assum}), this empirically supports the claim that DSCA edits
are structurally localized and exhibit minimal cross-concept
interference.

\paragraph{Discussion.}
The results above formalize the central intuition behind DSCA;
by decomposing the fused representation space into (approximately)
orthogonal concept-wise subspaces and restricting edits to operate
within these subspaces, the effect of an edit becomes spatially
localized in representation space.  As a result, preserving previously
learned concepts is no longer purely an algorithmic property (e.g.,
via replay or regularization), but also an \emph{architectural}
consequence of the subspace design.

% ================== ADDITIONAL METHODOLOGY ================== %
\section{Additional Methodology Details}
\label{sec:supp_method}

\subsection{Gating Implementation Details}
\label{sec:supp_gating_details}

As discussed in Sec.~\ref{sec:dsam}, the component-wise gating vector
$\gamma_k(\mathbf{h}_f) \in [0,1]^{d_f}$ is implemented via a lightweight neural layer
\[
\gamma_k(\mathbf{h}_f) = \sigma(W_{g,k}\mathbf{h}_f + b_{g,k}),
\]
where $\sigma$ is the element-wise sigmoid. To avoid quadratic parameter growth in $d_f$, we factorize $W_{g,k}$ as a low-rank bottleneck:
\[
W_{g,k} = U_k V_k,
\]
with $U_k \in \mathbb{R}^{d_f \times b}$, $V_k \in \mathbb{R}^{b \times d_f}$, and $b \ll d_f$.
This ensures that each DSAM adds only $O(d_f b)$ parameters, allowing DSCA to scale to many concepts without a prohibitive memory footprint.

\subsection{Exact Loss Definitions}
\label{sec:supp_losses}

For completeness, we provide the full mathematical form of the loss components introduced in Sec.~\ref{sec:loss}.
To navigate the trilemma of \textit{task fidelity}, \textit{locality}, and \textit{cross-modal alignment},(Sec.\ref{sec:intro}, \ref{sec:problem_formulation}) we design a composite loss function. Our training strategy operates on batches containing both new ``edit'' samples ($\mathcal{D}_e$) and ``out-of-scope'' replay samples ($\mathcal{D}_o$).
Each edit sample is denoted as $(I_e, T_e)$, representing the input image and text for which the model's behavior is being updated.
This dual-batch approach allows us to simultaneously learn the new task while preserving existing knowledge. The total loss is a weighted sum of four distinct objective terms.

\subsubsection{Task Fidelity Loss \texorpdfstring{($\mathcal{L}_{\text{task}}$)}{L_task}}
To ensure the edit is successful, we apply a standard causal language modeling loss to the edit batch $\mathcal{D}_e$. This loss minimizes the negative log-likelihood of the target token sequence $Y_{\text{target}} = (y_1, y_2, \dots, y_L)$, conditioned on the input and our model's intervention:
\begin{equation}
    \mathcal{L}_{\text{task}} = \mathbb{E}_{(\bm{I}_e, \bm{T}_e) \in \mathcal{D}_e} \left[ \sum_{i=1}^{L} \text{CrossEntropy}(\text{Logits}_i, y_i) \right],
\end{equation}
where $\text{Logits}_i$ is the model's predicted logit distribution for the $i$-th token after the update $h'_f$ has been applied. This directly optimizes the DSAM parameters to produce the correct textual output.

\subsubsection{Cross-Modal Alignment Loss \texorpdfstring{($\mathcal{L}_{\text{align}}$)}{L_align}}
An edit must not only produce the correct text but also maintain the VLM's fundamental cross-modal consistency. The update $\Delta\bm{h}_f$ is computed from the fused representation. To prevent this update from causing a modality drift, we enforce that the modified visual semantics remain coherent with the original textual semantics. We achieve this by regularizing the post-edit visual vector $\bm{h}'_{v,e} = \bm{h}_{v,e} + \Delta\bm{h}_{f,e}$ to remain aligned with the unmodified text vector $h_{t,e}$. This is formulated as a cosine similarity maximization, applied only to the edit batch $\mathcal{D}_e$:
\begin{equation}
    \mathcal{L}_{\text{align}} = \mathbb{E}_{(\bm{I}_e, \bm{T}_e) \in \mathcal{D}_e} \left[ 1 - \frac{(\bm{h}'_{v,e})^\top \bm{h}_{t,e}}{\|\bm{h}'_{v,e}\|_2 \|\bm{h}_{t,e}\|_2} \right].
\end{equation}
This asymmetric application i.e. updating the visual modality to match the static text—acts as a stable anchor, preventing the model from altering the textual concept space during a visually-driven edit.

\subsubsection{Contrastive Representation Distillation Loss \texorpdfstring{($\mathcal{L}_{\text{cdistill}}$)}{L_cdistill}}
To ensure locality, edits on one concept must not corrupt unrelated knowledge. A simple L2 distance loss on replay samples is insufficient as it penalizes all deviations equally and fails to preserve the \textit{relational geometry} of the embedding space. We therefore employ a more powerful contrastive distillation loss on the replay batch $\mathcal{D}_o$. For each sample in $\mathcal{D}_o$, the updated representation $\bm{h}'_{f,o}$ should be far more similar to its own original version (computed by a frozen teacher model, $\bm{h}^{\text{frozen}}_{f,o}$) than to any other sample in the batch. This is formulated using an InfoNCE objective \cite{oord2018representation}:
\begin{equation}
    \mathcal{L}_{\text{cdistill}} = -\mathbb{E}_{\bm{h}'_{f,o} \in \mathcal{D}_o} \left[ \log \frac{\exp(\text{sim}(\bm{h}'_{f,o}, \bm{h}^{\text{frozen}}_{f,o}) / \tau)}{\sum_{j \in \mathcal{D}_o} \exp(\text{sim}(\bm{h}'_{f,o}, \bm{h}^{\text{frozen}}_{f,o,j}) / \tau)} \right],
\end{equation}
where $\text{sim}(\cdot, \cdot)$ is the cosine similarity and $\tau$ is a temperature hyperparameter.
This loss provides a rich training signal that explicitly safeguards the model's relational knowledge structure against catastrophic forgetting.
%--------_%
\subsubsection{Gate Sparsity Loss \texorpdfstring{($\mathcal{L}_{\text{sparse}}$)}{L_sparse}}
To enforce a strict locality principle, interventions must only occur when an input is confidently matched to a concept. For out-of-scope samples, the routing weights $\{w_k\}$ should ideally all be zero. We encourage this behavior by applying an L1 penalty to the weights for samples in the replay batch $\mathcal{D}_o$. The L1 norm is well-suited for this task as it promotes solutions where most weights are exactly zero, preventing spurious DSAM activations.
\begin{equation}
    \mathcal{L}_{\text{sparse}} = \mathbb{E}_{\mathcal{D}_o} \left[ \sum_{k=1}^{K} |w_k| \right].
\end{equation}
This regularizer ensures that interventions are triggered selectively, preserving the integrity of unrelated representations.

%---------------%

\subsubsection{The Complete Objective Function}
The final loss is a weighted sum of these four components:
\begin{equation}
    \mathcal{L} = \mathcal{L}_{\text{task}} + \lambda_{\text{align}}\mathcal{L}_{\text{align}} + \lambda_{\text{distill}}\mathcal{L}_{\text{cdistill}} + \lambda_{\text{sparse}}\mathcal{L}_{\text{sparse}},
    \label{eq:total_loss_supp}
\end{equation}
where the $\lambda$ coefficients are hyperparameters that control the trade-off between plasticity and stability.

\subsection{Dual-Mode Parameter Update Framework}
\label{sec:dual_mode}

A key aspect of DSCA's stability is the separation of how different components are updated. Model parameters are partitioned into two categories governed by different update mechanisms and schedules. This framework corresponds to the DSCA Continual Training Loop and procedures outlined in Algorithm~\ref{alg:main_loop} and \ref{alg:procedures}.

\paragraph{Gradient-driven parameters (continuous update).}
This group contains all parameters directly responsible for executing the intervention. For each concept $k$ (Sec.~\ref{sec:concept_partitioning}), this set includes the subspace transformation parameters $\{\mathbf{W}_k, \mathbf{b}_k\}$ and gating parameters $\{\mathbf{W}_{g,k}, \mathbf{b}_{g,k}\}$ (Sec.~\ref{sec:dsam}). These parameters are learnable via backpropagation and are updated at every training step using gradients of the total loss $\mathcal{L}$ (Eq.~\ref{eq:total_loss}).

\paragraph{Data-driven structural components (periodic update).}
This group contains the architectural backbone of the knowledge base; concept prototypes $\{\mathbf{p}_{k,v}, \mathbf{p}_{k,f}\}$ (Sec.~\ref{sec:HierarchicalGating}, \ref{sec:concept_partitioning}) and semantic subspace bases $\{\mathbf{R}_k\}$ (Sec.~\ref{sec:dsam}). These components are not updated via backpropagation but follow a data-dependent schedule:
\begin{itemize}
    \item \textbf{Prototypes} ($\mathbf{p}_k$) are updated via Exponential Moving Average (EMA) whenever a new sample is assigned to their cluster.
    \item \textbf{Subspaces} ($\mathbf{R}_k$) follow a lifecycle based on data accumulation. A subspace is initialized via PCA only after its corresponding cluster buffer accumulates at least \textbf{$N_{\text{min}}$} samples, ensuring sufficient statistics to form valid principal components. Once active, the basis is refined periodically (every \textbf{$N_{\text{refine}}$} steps) using Incremental PCA on buffered features, with residualization across subspaces to maintain approximate orthogonality.
\end{itemize}
This dual-mode update separates rapid, fine-grained learning of ``how to act'' (DSAM parameters) from slower, structural re-organization of ``how to represent concepts'' (subspaces and prototypes), which is crucial for long-term stability and adaptability.

% ================== ADDITIONAL EXPERIMENTAL RESULTS ================== %

\begin{figure}[t]
    \centering
    \includegraphics[width=0.9\linewidth]{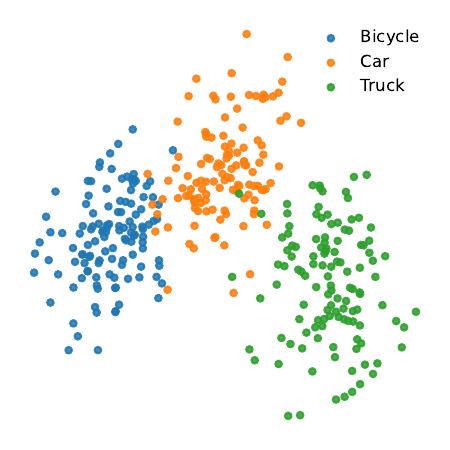}
    \vspace{-2mm}
    \caption{\textbf{Concept-wise subspace visualization.}
    t-SNE of fused representations projected through their assigned semantic subspaces $\mathbf{R}_k^\top \mathbf{h}_f$ for a subset of concepts (e.g., ``car'', ``truck'', ``bicycle''). DSCA yields compact, well-separated clusters, indicating that edits remain confined to localized regions of the representation space.This empirically validates the conceptual illustration provided in Figure \ref{fig:overview}(d) of the main paper.}
    \label{fig:dsca_subspace_tsne}
    \vspace{-3mm}
\end{figure}

% ================== METRICS ================== %
\section{Evaluation Metrics}
\label{sec:eval_metrics}

In this section, we provide the formal definitions of all evaluation metrics referenced in Sec.~\ref{sec:metrics} of the main paper. Let $f_{\theta_0}$ denote the original (unedited) model and $f_{\theta_t}$ the model after $t$ sequential edits. Each edit request is represented as a tuple $e = (v, p, o)$, consisting of a visual input $v$, textual prompt $p$, and desired target output $o$. The set of edits up to step $t$ is denoted by $\mathcal{D}_{e,t}$. We use $\mathbb{I}\{\cdot\}$ for the indicator function.

\subsection{Editing Metrics}

To provide a comprehensive assessment of model editing performance, we evaluate each method along three fundamental axes: Efficacy, Generalization, and Locality. We formalize these metrics below. Please  Let $f_{\theta_0}$ be the original, unedited model and $f_{\theta_t}$ be the model after $t$ sequential edits. An edit request is a tuple $e=(v, p, o)$, consisting of a visual input $v$, a textual prompt $p$, and the desired target output $o$. Let $\mathcal{D}_{e,t}$ be the set of $t$ edits applied to the model. We use $\mathbb{I}\{\cdot\}$ to denote the indicator function, which equals 1 if its condition is true, and 0 otherwise.

\paragraph{Efficacy:}This measures the direct success of the edits.
\begin{itemize}
    \item \textbf{Reliability (Rel.)} quantifies whether the model $f_{\theta_t}$ correctly produces the target output for all edit examples it has been trained on up to step $t$.
\begin{flalign}
\label{eq:reliability}
\text{Rel.} = \mathbb{E}_{(v, p, o) \sim \mathcal{D}_{e,t}} \left[ \mathbb{I}\{ f_{\theta_t}(v, p) = o \} \right]
\end{flalign}
\end{itemize}

\paragraph{Generalization:}This assesses whether the edit was learned as a robust concept, rather than being merely memorized for one specific input.To test this, we evaluate the model on a set of inputs that are semantically equivalent but vary in their presentation. For a given input, we denote this set of variations as $\mathcal{N}(\cdot)$ and samples from it with a subscript $g$.
\begin{itemize}
    \item \textbf{Textual Generalization (T-Gen.)} measures robustness to linguistic variation by evaluating paraphrases ($p_g \in \mathcal{N}(p)$) of the original prompt.
    \begin{equation}
    \label{eq:tgen}
    \text{T-Gen.} = \mathbb{E}_{(v, p, o) \sim \mathcal{D}_{e,t}} \mathbb{E}_{p_g \sim \mathcal{N}(p)} \left[ \mathbb{I}\{ f_{\theta_t}(v, p_g) = o \} \right]
    \end{equation}
    
    \item \textbf{Visual Generalization (V-Gen.)} measures robustness to visual variation by evaluating on images ($v_g \in \mathcal{N}(v)$) that depict the same subject but from different viewpoints or lighting conditions. This metric is also referred to as Modal Generality (M-Gen) in Table~\ref{tab:lifelong}.
    \begin{equation}
    \label{eq:vgen}
    \text{V-Gen.} = \mathbb{E}_{(v, p, o) \sim \mathcal{D}_{e,t}} \mathbb{E}_{v_g \sim \mathcal{N}(v)} \left[ \mathbb{I}\{ f_{\theta_t}(v_g, p) = o \} \right]
    \end{equation}
\end{itemize}

\paragraph{Locality:} This ensures that edits do not degrade performance on unrelated inputs by comparing the behavior of $f_{\theta_t}$ to the original model $f_{\theta_0}$. This is measured using a set of inputs that are unrelated to the edit context. We denote this set as $\mathcal{U}(\cdot)$ and samples from it with a subscript $u$.
\begin{itemize}
    \item \textbf{Textual Locality (T-Loc.)} evaluates on unrelated, text-only inputs ($p_u \in \mathcal{U}(p)$).
\begin{flalign}
\label{eq:tloc}
\text{T-Loc.} =\;
&\mathbb{E}_{(v,p,o) \sim \mathcal{D}_{e,t}}
\mathbb{E}_{p_u \sim \mathcal{U}(p)} \nonumber \\
&\left[
\mathbb{I}\{ f_{\theta_t}(\emptyset, p_u)
= f_{\theta_0}(\emptyset, p_u) \}
\right] &&
\end{flalign}

    \item \textbf{Multimodal Locality (M-Loc.)} evaluates on unrelated multimodal inputs ($(v_u, p_u) \in \mathcal{U}(v, p)$).
\begin{flalign}
\label{eq:mloc}
\text{M-Loc.} =\;
&\mathbb{E}_{(v,p,o) \sim \mathcal{D}_{e,t}} \,
\mathbb{E}_{(v_u,p_u) \sim \mathcal{U}(v,p)} \nonumber \\
&\left[
\mathbb{I}\{ f_{\theta_t}(v_u, p_u)
= f_{\theta_0}(v_u, p_u) \}
\right] &&
\end{flalign}
\end{itemize}

\paragraph{Aggregate Score:}Finally, to summarize overall performance, we compute the average score.
\begin{itemize}
    \item \textbf{Average (Avg.)} provides a single, holistic measure by calculating the arithmetic mean of the five core metrics: Reliability, Textual Generalization, Visual Generalization, Textual Locality, and Multimodal Locality.
\end{itemize}
\begin{table*}[htbp]
    \centering
    \caption{DSCA-specific hyperparameters used in all experiments. Values are kept fixed across editing and continual learning unless otherwise noted.}
    \label{tab:dsca_hparams}
    % \resizebox{\textwidth}{!}{%
    \begin{tabular}{lccccccc}
        \toprule
        Setting &
        $r$ &
        $N_{\text{min}}$ &
        $N_{\text{refine}}$ &
        $\tau$ &
        $\lambda_{\text{align}}$ &
        $\lambda_{\text{distill}}$ &
        $\lambda_{\text{sparse}}$ \\
        \midrule
        LLaVA-1.5-7B\cite{liu2023improvedllava} (editing: E-VQA / E-IC / VLKEB) &
        128 & 32 & 500 & 0.07 & 0.5 & 1.0 & $1\times10^{-2}$ \\
        PaliGemma-3B\cite{beyer2024paligemma} (CoIN continual learning\cite{chen2024coin}) &
        128 & 32 & 500 & 0.07 & 0.5 & 1.0 & $1\times10^{-2}$ \\
        \bottomrule
    \end{tabular}%
    %}
\end{table*}
\subsection{Continual Learning Metrics}

% To evaluate performance on the CoIN benchmark, we adopt standard continual-learning metrics. Let $T$ be the number of tasks. Denote $A_{t,i}$ as the accuracy on task $i$ after training sequentially on tasks $1,\dots,t$, and $\bar{a}_i$ as the initial zero-shot accuracy on task $i$.

% \paragraph{Average Accuracy (ACC).}
% Measures final performance across all tasks after completing the sequence:
% \begin{equation}
% \text{ACC} = \frac{1}{T} \sum_{i=1}^{T} A_{T,i}.
% \end{equation}

% \paragraph{Backward Transfer (BWT).}
% Quantifies how learning new tasks influences performance on past tasks. Negative values indicate forgetting:
% \begin{equation}
% \text{BWT} = \frac{1}{T-1} \sum_{i=1}^{T-1} (A_{T,i} - A_{i,i}).
% \end{equation}

% \paragraph{Forward Transfer (FWT).}
% Measures how knowledge from previous tasks improves performance on future tasks before training on them:
% \begin{equation}
% \text{FWT} = \frac{1}{T-1} \sum_{i=2}^{T} (A_{i-1,i} - \bar{a}_i).
% \end{equation}

% \paragraph{Average Task Accuracy ($A_t$).}
% Captures model plasticity via performance immediately after learning each task:
% \begin{equation}
%  A_t = \frac{1}{T} \sum_{i=1}^{T} A_{i,i}.
% \end{equation}
To evaluate the performance of our method on the CoIN continual learning benchmark, we employ a set of standard metrics designed to measure catastrophic forgetting, knowledge transfer, and plasticity. Below, we provide the formal definitions for these metrics.

Let $T$ be the total number of tasks in the continual learning sequence. We denote $A_{t,i}$ as the accuracy of the model on task $i$ after it has been trained sequentially on tasks $1, 2, \ldots, t$. Consequently, $A_{i,i}$ is the accuracy on task $i$ immediately after training on it, and $\bar{a}_i$ is the initial zero-shot accuracy on task $i$ before any fine-tuning.

\subsubsection{Average Accuracy (ACC)}
Average Accuracy measures the final, overall performance of the model across all tasks after the entire sequence of $T$ tasks has been learned. It provides a single score to summarize how well the model has retained knowledge and learned all tasks by the end.

ACC is calculated by averaging the final accuracies on each task $i$ after the model has been trained on all $T$ tasks.
\begin{equation}
\text{ACC} = \frac{1}{T} \sum_{i=1}^{T} A_{T,i}
\end{equation}

\subsubsection{Backward Transfer (BWT)}
Backward Transfer measures the influence of learning a new task on the performance of previously learned tasks. It is the primary metric for quantifying catastrophic forgetting. A negative BWT value indicates forgetting, while a value close to zero implies knowledge retention.

BWT is calculated as the average difference between the final accuracy on a past task $i$ ($A_{T,i}$) and the accuracy it had immediately after it was first learned ($A_{i,i}$).
\begin{equation}
\text{BWT} = \frac{1}{T-1} \sum_{i=1}^{T-1} (A_{T,i} - A_{i,i})
\end{equation}

\subsubsection{Forward Transfer (FWT)}
Forward Transfer measures the model's ability to generalize from past tasks to improve its learning on future, unseen tasks. It quantifies how much the knowledge gained from learning tasks $1$ to $t-1$ helps the model's performance on the next task $t$ before it is explicitly trained on it. A positive FWT is desirable.

FWT is calculated by averaging the difference between the model's accuracy on a new task $i$ before training on it ($A_{i-1,i}$) and its initial zero-shot accuracy on that same task ($\bar{a}_i$).
\begin{equation}
\text{FWT} = \frac{1}{T-1} \sum_{i=2}^{T} (A_{i-1,i} - \bar{a}_i)
\end{equation}

\subsubsection{Average Task Accuracy ($A_t$)}
Average Task Accuracy assesses the model's plasticity—its ability to effectively learn each new task. It measures the peak performance achieved on each task right after its respective training phase is complete, without considering whether that knowledge is later forgotten.

It is the average of the accuracies on each task $i$ measured immediately after the model has finished training on that task.
\begin{equation}
 A_t = \frac{1}{T} \sum_{i=1}^{T} A_{i,i}
\end{equation}

% ================== IMPLEMENTATION DETAILS ================== %
\section{Implementation Details and Hyperparameters}
\label{sec:implementation_details}

In this section, we detail the experimental setup and hyperparameter configurations used to train and evaluate DSCA. All experiments were conducted using the PyTorch framework on $8\times$ NVIDIA A100 (80GB) GPUs with mixed-precision training.

\vspace{5pt}
\noindent\textbf{Backbone Models.} We apply DSCA to two distinct vision-language architectures to demonstrate generality:
\begin{itemize}
    \item \textbf{LLaVA-1.5-7B \cite{liu2023improvedllava}:} Used for all model editing benchmarks (E-VQA\cite{cheng2023edit}, E-IC\cite{cheng2023edit}, VLKEB\cite{huang2024vlkeb}).
    \item \textbf{PaliGemma-3B \cite{beyer2024paligemma}:} Used for the CoIN continual learning benchmark\cite{chen2024coin}.
\end{itemize}

\vspace{5pt}
\noindent\textbf{Hyperparameter Configuration.} To ensure reproducibility and fair comparison, we maintain a fixed set of hyperparameters across all benchmarking experiments. These specific values are listed in \textbf{Table \ref{tab:dsca_hparams}}. 

\noindent Key hyperparameters include:
\begin{itemize}
    \item \textbf{Subspace Rank ($r$):} The dimensionality of the semantic subspaces ($r=128$). (Sec.\ref{sec:dsam})
    \item \textbf{Subspace Initialization Threshold ($N_{\text{min}}$):} The minimum number of samples a concept cluster must accumulate before its semantic subspace basis $\mathbf{R}_k$ is initialized via PCA and the corresponding DSAM becomes active ($N_{\text{min}}=32$). (Algo.\ref{alg:procedures}, Sec.\ref{sec:dual_mode})
    \item \textbf{Refinement Interval ($N_{\text{refine}}$):} The frequency of Incremental PCA updates for the semantic bases ($N_{\text{refine}}=500$ steps).(Algo.\ref{alg:main_loop}, Sec.\ref{sec:dual_mode})
    \item \textbf{Routing Temperature ($\tau$):} Controls the sharpness of the routing distribution ($\tau=0.07$). (Sec.\ref{sec:HierarchicalGating})
    \item \textbf{Loss Coefficients ($\lambda$):} We balance the objective using $\lambda_{\text{align}}=0.5$ (cross-modal alignment), $\lambda_{\text{distill}}=1.0$ (contrastive distillation), and $\lambda_{\text{sparse}}=1 \times 10^{-2}$ (gate sparsity). (Eq.\ref{eq:total_loss}). This value of $\lambda_{\text{sparse}}=1 \times 10^{-2}$ is reflected as the
    Operating Point in Graph \ref{fig:routsparse}(b)
\end{itemize}
Unless otherwise noted in the specific experimental subsection, these values remain constant throughout the lifelong editing and continual learning processes.

\section{Extended Experimental Results}
\label{sec:extended_results}

We provide expanded comparisons against a wider range of baselines in Tables~\ref{tab:Dualedit-extended}, \ref{tab:lifelong_extended}, and \ref{tab:PAM_extended}.

\subsection{Expanded Single-Edit Performance }
Table \ref{tab:Dualedit-extended} provides a comprehensive single edit success comparison on the E-VQA and E-IC benchmarks. 

    All baseline numbers, including standard fine-tuning variants and retrieval-based methods, are sourced from \textbf{DualEdit \cite{shi2024dualedit}}.
    Results for \textbf{FT-V} (Cheng et al., 2023), \textbf{FT-L} (Cheng et al., 2023), \textbf{KE} (De Cao et al., 2021), \textbf{IKE} (Li et al., 2023), \textbf{TP} (Huang et al., 2023), and \textbf{VisEdit}\cite{chen2025attribution} should be cross-referenced with the experiment tables provided in \textbf{DualEdit \cite{shi2024dualedit}}.

\subsection{Expanded Lifelong Editing Performance }
Table \ref{tab:lifelong_extended} details performance after $t=1,000$ sequential edits on LLaVA-1.5-7B.

     All baseline numbers are sourced from \textbf{LiveEdit \cite{chen2025lifelong}}.
Results for \textbf{RECIPE}(Chen et al. 2024), \textbf{LEMoE}(Wang et al. 2024), \textbf{FT-M} (Cheng et al. 2023), and \textbf{FT-L} (Cheng et al. 2023) are taken directly from the Sequential Editing benchmarks reported in \textbf{LiveEdit \cite{chen2025lifelong}}.

\subsection{Expanded Continual Learning on CoIN }
Table \ref{tab:PAM_extended} reports results on the CoIN benchmark using the PaliGemma-3B backbone\cite{beyer2024paligemma}.All baseline numbers are sourced from \textbf{PAM \cite{sokar2025continual}}.
To establish performance bounds, we include three foundational setups defined in their work:

\begin{itemize}
    \item \textbf{Zero-shot:} Evaluates the pre-trained PaliGemma base model directly without further training. This assesses the model's inherent generalization capabilities prior to task-specific adaptation.
    \item \textbf{Independent:} Fine-tunes a separate LoRA module for each task individually. This serves as an upper bound for task-specific performance by completely isolating tasks to prevent interference, though it requires maintaining distinct parameters for every task.
    \item \textbf{Multitask:} Trains a single LoRA module on all tasks concurrently. This represents the theoretical upper bound for a single shared model by assuming simultaneous access to all datasets, though it violates the sequential data constraints of Continual Learning.
\end{itemize}
Results for \textbf{LwF} (Li and Hoiem 2017), \textbf{I-LoRA} (Li et al. 2025), \textbf{O-LoRA} (Wang et al. 2023), \textbf{MoELoRA}\cite{chen2024coin}, and \textbf{MagMax \cite{marczak2024magmax}} are cited as reported in the Experiments section of \textbf{PAM \cite{sokar2025continual}}. Please refer to \textbf{PAM \cite{sokar2025continual}} for the specific implementation details of these continual learning baselines.
\FloatBarrier
\begin{table*}[t]
\centering
\caption{Extended Single Edit success comparision. Baselines from DualEdit\cite{shi2024dualedit}}
\label{tab:Dualedit-extended}
\begin{tabular}{l rrrrrr rrrrrr}
\toprule
\multirow{2}{*}{\textbf{Methods}} & \multicolumn{6}{c}{\textbf{E-VQA\cite{ cheng2023edit}}} & \multicolumn{6}{c}{\textbf{E-IC\cite{ cheng2023edit}}} \\
\cmidrule(lr){2-7} \cmidrule(lr){8-13}
& \textbf{Rel.} & \textbf{T-Gen.} & \textbf{V-Gen.} & \textbf{T-Loc.} & \textbf{M-Loc.} & \textbf{Avg.} & \textbf{Rel.} & \textbf{T-Gen.} & \textbf{V-Gen.} & \textbf{T-Loc.} & \textbf{M-Loc.} & \textbf{Avg.} \\
\midrule
\multicolumn{13}{l}{\textit{LLaVA-V1.5 \cite{liu2023improvedllava}}} \\
FT-V & 31.68 & 29.96 & 26.68 & 100.00 & 91.23 & 55.91 & 52.85 & 51.57 & 48.63 & 100.00 & 92.55 & 69.12 \\
FT-L & 31.78 & 30.02 & 26.91 & 99.94 & 92.03 & 56.14 & 53.00 & 51.02 & 49.29 & 98.91 & 94.89 & 69.42 \\
KE  & 85.86 & 84.00 & 82.23 & 93.57 & 73.06 & 83.74 & 83.54 & 82.15 & 81.12 & 92.46 & 73.83 & 82.62 \\
IKE & 91.35 & 90.84 & 91.08 & 60.18 & 51.08 & 76.91 & 93.72 & 88.37 & 76.99 & 76.60 & 64.90 & 80.12 \\
SERAC \cite{mitchell2022memory} & 82.51 & 81.60 & 80.05 & 100.00 & 57.48 & 80.33 & 43.08 & 42.37 & 42.85 & 100.00 & 7.63 & 47.19 \\
MEND \cite{mitchell2022fast} & 92.30 & 92.16 & 92.10 & 90.30 & 81.13 & 89.60 & 93.76 & 93.46 & 92.14 & 91.60 & 87.59 & 91.71 \\
TP & 38.68 & 36.27 & 31.26 & 95.31 & 91.41 & 58.59 & 59.07 & 57.01 & 55.51 & 64.79 & 89.26 & 65.13 \\
LTE \cite{jiang2024learning} & 94.16 & 93.54 & 93.06 & 83.76 & 81.65 & 89.23 & 93.60 & 92.38 & 91.18 & 85.54 & 88.49 & 90.24 \\
VisEdit \cite{chen2025attribution} & 95.78 & 94.21 & 94.37 & 100.00 & 91.11 & 95.09 & 95.06 & 94.87 & 94.35 & 100.00 & 95.23 & 95.90 \\
DualEdit \cite{shi2024dualedit}& 96.94 & 96.43 & 96.20 & 100.00 & 99.61 & 97.84 & 96.76 & 96.52 & 96.24 & 100.00 & 99.74 & 97.85 \\
DSCA (Ours) & \textbf{98.12} & \textbf{97.30} & \textbf{97.25} & \textbf{100.00} & \textbf{99.83} & \textbf{98.50} & \textbf{98.00} & \textbf{97.10} & \textbf{97.02} & \textbf{100.00} & \textbf{99.90} & \textbf{98.00} \\
\bottomrule
\end{tabular}%
\end{table*}

\begin{table*}[t]
\centering
\caption{Expanded Lifelong editing results ($t=1000$ edits) on LLaVA-1.5-7B \cite{liu2023improvedllava}. Baselines from LiveEdit~\cite{chen2025lifelong}.}
\label{tab:lifelong_extended}
\begin{tabular}{l rrrrrr rrrrrr}
\toprule
\multirow{2}{*}{\textbf{Methods}} & \multicolumn{6}{c}{\textbf{E-VQA\cite{ cheng2023edit}}} & \multicolumn{6}{c}{\textbf{VLKEB}\cite{huang2024vlkeb}} \\
\cmidrule(lr){2-7} \cmidrule(lr){8-13}
& \textbf{Rel.} & \textbf{T-Gen.} & \textbf{V-Gen.} & \textbf{T-Loc.} & \textbf{M-Loc.} & \textbf{Avg.} & \textbf{Rel.} & \textbf{T-Gen.} & \textbf{V-Gen.} & \textbf{T-Loc.} & \textbf{M-Loc.} & \textbf{Avg.} \\
\midrule
\multicolumn{13}{l}{\textit{LLaVA-V1.5 \cite{liu2023improvedllava}}} \\
LiveEdit\cite{chen2025lifelong} & 92.93 & 90.16 & 84.30 & 100.00 & 96.43 & 92.76 & 92.22 & 83.97 & 82.75 & 100.00 & 100.00 & 91.79 \\
LTE \cite{jiang2024learning}    & 83.93 & 82.55 & 81.34 & 83.97 & 73.09 & 80.98 & 64.51 & 56.26 & 64.80 & 80.85 & 76.52 & 68.59 \\
MEND\cite{mitchell2022fast}     & 0.04  & 0.05  & 0.05  & 0.08   & 0.09  & 0.06  & 0.03  & 0.05  & 0.07  & 0.06   & 0.08   & 0.06 \\
SERAC \cite{mitchell2022memory} & 85.57 & 75.58 & 82.01 & 62.46  & 15.69 & 64.26 & 60.93 & 56.49 & 60.06 & 52.94  & 15.04  & 49.09 \\
FT-L  & 71.39 & 59.83 & 57.41 & 55.55 & 48.99 & 58.63  & 68.14 & 66.38 & 66.98 & 65.61 & 75.35 & 68.49  \\
FT-M  & 69.57 & 56.34 & 44.07 & 100.00 & 41.47 & 62.29 & 53.41 & 48.80 & 43.16 & 100.00 & 57.03 & 60.48 \\
TP & 16.56 & 16.80 & 15.65 & 7.28 & 15.60 & 14.38  & 5.46 & 4.81 & 5.51 & 2.77 & 7.19 & 5.15 \\
RECIPE  & 87.00 & 76.81 & 83.09 & 86.95 & 87.03 & 84.18  & 62.00 & 56.84 & 61.50 & 85.37 & 82.07 & 69.56  \\
LEMoE & 30.80 & 25.75 & 24.32 & 71.45 & 46.23 & 39.71  & 67.97 & 61.07 & 58.16 & 48.48 & 44.06 & 55.95  \\
\textbf{DSCA (ours)}            & \textbf{96.85} & \textbf{93.10} & \textbf{88.00} & \textbf{100.00} & \textbf{98.20} & \textbf{95.23} & \textbf{98.10} & \textbf{93.80} & \textbf{89.70} & \textbf{100.00} & \textbf{100.00} & \textbf{96.72} \\
\bottomrule
\end{tabular}%

\end{table*}
\begin{table*}[t]
    \centering
    \caption{Expanded table for perfomance on COIN benchmark\cite{chen2024coin}. Higher is better for all metrics; less negative BWT indicates less forgetting.Baselines are taken from PAM\cite{sokar2025continual}}
    \label{tab:PAM_extended}
% \resizebox{\textwidth}{!}{%
    \begin{tabular}{lcccc}
        \toprule
        \textbf{METHOD} & \textbf{ACC} & \textbf{BWT} & \textbf{FWT} & \textbf{$A_t$} \\
        \midrule
        ZERO-SHOT   & 24.74          & -              & -              & -              \\
        INDEPENDENT & 76.46          & -              & -              & -              \\
        MULTITASK   & 73.93          & -              & -              & -              \\
        \midrule
        FINE-TUNE   & 43.36$\pm$8.18 & -39.51$\pm$9.87 & 7.71$\pm$2.51  & 76.29$\pm$0.18 \\
        LWF       & 47.15$\pm$3.52 & -33.66$\pm$3.84 & 9.67$\pm$1.22  & 75.20$\pm$0.32 \\
        I-LoRA    & 42.11$\pm$6.55 & -40.95$\pm$8.13 & 7.89$\pm$2.60  & 76.24$\pm$0.35 \\
         O-LoRA     & 46.53$\pm$6.88 & -32.08$\pm$7.87 & 9.45$\pm$3.02  & 73.27$\pm$1.01 \\
        MoELoRA \cite{chen2024coin}    & 46.59$\pm$9.98 & -36.40$\pm$11.97& 7.79$\pm$2.24  & \textbf{76.93$\pm$0.27} \\
        MagMax\cite{marczak2024magmax}      & 45.74$\pm$0.88 & -22.68$\pm$6.51 & 4.75$\pm$3.36  & 76.29$\pm$0.18 \\
        PAM\cite{sokar2025continual}  & 49.89$\pm$1.66 & -19.45$\pm$0.95 & \textbf{11.11$\pm$0.09} & 76.31$\pm$0.03 \\
        DSCA (ours,with PaliGemma-3B\cite{beyer2024paligemma}) & \textbf{49.96$\pm$0.72} & \textbf{-9.37$\pm$1.02} & 11.04$\pm$0.13 & 76.48$\pm$0.07 \\
        \bottomrule
    \end{tabular}
\end{table*}

\end{document}